\documentclass[nohyperref]{article}

\usepackage{microtype}
\usepackage{graphicx}
\usepackage{subfigure}
\usepackage{booktabs} % for professional tables

\usepackage{hyperref}
\usepackage{enumitem}
\usepackage{multirow}
\usepackage{hhline}
\usepackage{pifont}
\usepackage{amsmath}
\usepackage{bm}
\usepackage{array}
\usepackage[accepted]{icml2022}

\usepackage{amsmath}
\usepackage{amssymb}
\usepackage{mathtools}
\usepackage{amsthm}

\usepackage[capitalize,noabbrev]{cleveref}

\theoremstyle{plain}
\newtheorem{theorem}{Theorem}[section]

\newtheorem{lemma}[theorem]{Lemma}

\theoremstyle{definition}
\newtheorem{definition}[theorem]{Definition}

\theoremstyle{remark}

\usepackage[textsize=tiny]{todonotes}

\long\def\comment#1{}

\newcommand{\rom}[1]{\uppercase\expandafter{\romannumeral #1\relax}}

\newcommand{\ada}{GloGNN}

\icmltitlerunning{Submission and Formatting Instructions for ICML 2022}

\begin{document}

\twocolumn[
\icmltitle{Finding Global Homophily in Graph Neural Networks When Meeting Heterophily}

% It is OKAY to include author information, even for blind
% submissions: the style file will automatically remove it for you
% unless you've provided the [accepted] option to the icml2022
% package.

% List of affiliations: The first argument should be a (short)
% identifier you will use later to specify author affiliations
% Academic affiliations should list Department, University, City, Region, Country
% Industry affiliations should list Company, City, Region, Country

% You can specify symbols, otherwise they are numbered in order.
% Ideally, you should not use this facility. Affiliations will be numbered
% in order of appearance and this is the preferred way.
\icmlsetsymbol{equal}{*}

\begin{icmlauthorlist}
\icmlauthor{Xiang Li}{yyy}
\icmlauthor{Renyu Zhu}{yyy}
\icmlauthor{Yao Cheng}{yyy}
\icmlauthor{Caihua Shan}{comp}
\icmlauthor{Siqiang Luo}{sch}
\icmlauthor{Dongsheng Li}{comp}
\icmlauthor{Weining Qian}{yyy}
%\icmlauthor{}{sch}
%\icmlauthor{Firstname8 Lastname8}{sch}
%\icmlauthor{Firstname8 Lastname8}{yyy,comp}
%\icmlauthor{}{sch}
%\icmlauthor{}{sch}
\end{icmlauthorlist}

\icmlaffiliation{yyy}{School of Data Science and Engineering, East China Normal University, Shanghai, China}
\icmlaffiliation{comp}{Microsoft Research Asia, Shanghai, China}
\icmlaffiliation{sch}{School of Computer Science and Engineering, Nanyang Technological University, Singapore}

\icmlcorrespondingauthor{Xiang Li}{xiangli@dase.ecnu.edu.cn}
%\icmlcorrespondingauthor{Firstname2 Lastname2}{first2.last2@www.uk}

% You may provide any keywords that you
% find helpful for describing your paper; these are used to populate
% the "keywords" metadata in the PDF but will not be shown in the document
\icmlkeywords{Machine Learning, ICML}

\vskip 0.3in
]

% this must go after the closing bracket ] following \twocolumn[ ...

% This command actually creates the footnote in the first column
% listing the affiliations and the copyright notice.
% The command takes one argument, which is text to display at the start of the footnote.
% The \icmlEqualContribution command is standard text for equal contribution.
% Remove it (just {}) if you do not need this facility.

\printAffiliationsAndNotice{}  % leave blank if no need to mention equal contribution
%\printAffiliationsAndNotice{\icmlEqualContribution} % otherwise use the standard text.

\begin{abstract}
%Graph neural networks on graphs with heterophily
%have recently received wide attention.
We investigate graph neural networks on graphs with heterophily.
Some existing methods 
%address the heterophily issue
amplify a node's neighborhood with multi-hop neighbors
to include more nodes with homophily.
%with the same label as the node.
However,
it is a significant challenge to set personalized neighborhood sizes for different nodes.
%how large the neighborhood size should be set for different nodes is a challenge.
%On the other hand,
Further,
for other homophilous nodes excluded in the neighborhood,
they are ignored for information aggregation.
To address these problems,
we propose two models \ada\ and \ada++,
which generate a node's embedding by aggregating information from global nodes in the graph.
%to utilize more neighbors in the same class with the node.
%\ada\ automatically combines low-pass and high-pass convolutional filters
%by learning a coefficient matrix that captures the correlations between nodes and allows signed values.
In each layer,
both models learn a coefficient matrix to capture the correlations between nodes,
based on which neighborhood aggregation is performed.
The coefficient matrix allows signed values and is derived from an optimization problem that has a closed-form solution.
We further accelerate neighborhood aggregation 
and derive a linear time complexity.
%we avoid directly computing the coefficient matrix and reorder matrix multiplication 
%to reduce the time complexity from quadratic to linear.
%which allows signed values.
%After that,
%the model performs neighborhood aggregation based on the matrix,
%which automatically combines low-pass and high-pass convolutional filters.
%We further speed up the aggregation with matrix multiplication reordering. 
%To explain the effectiveness of the model,
We theoretically explain the models' effectiveness by proving that both the coefficient matrix 
and the generated node embedding matrix have the desired grouping effect.
%We also put forward an upgraded model \ada++,
%which learns the importance of hidden node features.
We conduct extensive experiments to 
compare our models against 11 other competitors on 15 benchmark datasets in a wide range of 
domains, scales and graph heterophilies.
Experimental results show that
our methods achieve superior performance and are also very efficient.
\end{abstract}

\section{Introduction}
\label{sec:intro}
Graph-structured data is ubiquitous 
in a variety of domains including chemistry, biology and sociology.
%graphs are widely used to model structured data 
In graphs (networks),
%\footnote{We interchangeably use graph and network in this paper.}, 
nodes and edges represent entities and their relations, respectively. %represent relations between entities.
To enrich the information of graphs, 
nodes are usually associated with various features.
For example,
on Facebook,
users are connected by the \emph{friendship} relation and each user has features like \emph{age}, \emph{gender} and \emph{school}.
Both node features and graph topology provide sources of information for graph-based learning.
Recently,
graph neural networks (GNNs)~\cite{kipf2016semi,velivckovic2017graph,hamilton2017inductive}
have received significant attention for the capability to seamlessly integrate the two sources of information
and they have been shown to serve as
effective tools for representation learning on graph-structured data.

%Graph neural networks (GNNs) are effective tools for representation learning on graph structured data.
%By seamlessly utilizing node features and network\footnote{We interchangeably use graph and network in this paper} structure,
%GNNs have been applied in many downstream tasks, such as node classification, link prediction and recommendation.
%Assuming the homophily property of graphs,
Based on the implicit graph homophily assumption,
traditional GNNs~\cite{kipf2016semi}
adopt a non-linear form of smoothing operation %on the neighborhood 
and 
generate node embeddings by aggregating information from a node's neighbors.
Specifically,
\emph{homophily} is a key characteristic in a wide range of real-world graphs, 
where linked nodes tend to share similar features or have the same label.
These graphs include friendship networks~\cite{mcpherson2001birds}, political networks~\cite{gerber2013political} and citation networks~\cite{ciotti2016homophily}. 
However, 
in the real world,
there also exist graphs with \emph{heterophily},
%\footnote{In this paper, we study heterophily in terms of node labels.},
where nodes with dissimilar features or different labels are more likely to be connected.
For example,
different amino acid types are connected in protein structures;
%words with opposite meanings are linked by the antonym relation in the word network constructed from WordNet;
predator and prey are related in the ecological food webs.
In these networks,
due to the heterophily,
the smoothing operation could generate similar representations for nodes with different labels,
which lead to the poor performance of 
%over a node's neighborhood could nodes with different labels,
GNNs. 
%thus have been shown to perform poorly.

\begin{figure}[!htbp]
    \centering
        \includegraphics[width = 0.7\linewidth]{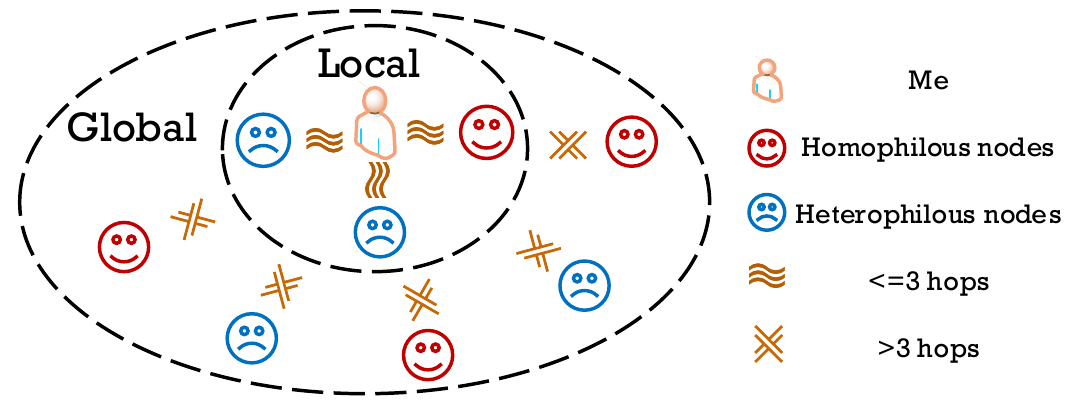}
        \caption{A toy example to show global homophily. All the homophilous nodes express the global homophily of the center user.}
        \label{fig:toy}
\end{figure}

To generalize GNNs to heterophilous graphs,
some recent works~\cite{zhu2020beyond,bo2021beyond,chien2020adaptive} have been proposed
to leverage high-pass convolutional filters and multi-hop neighbors to address the heterophily issue. 
On the one hand,
the high-pass filters can be used to push away a node's feature vector from its neighbors' while 
the low-pass filters used by traditional GNNs 
do the opposite.
The combination of low-pass and high-pass filters in these models 
enforces the learned representation of a node 
to be close to its homophilous neighbors' and distant from heterophilous ones'.
%We refer to these homophilous neighbors as a node's ``\emph{friends}''.
%In graphs with heterophily, high-pass filters are thus important.
%This explains the usefulness of high-pass filters in graphs with heterophily.
On the other hand,
in heterophilous graphs,
linked nodes are more likely to be dissimilar while distant nodes could share a certain similarity,
so we have to jump the \emph{locality} of a node 
to find its homophilous neighbors.
As shown in Figure~\ref{fig:toy},
a user has only one \emph{local} neighbor with homophily,
while three homophilous nodes exist multi-hop away.
All these four nodes 
exhibit the \emph{global homophily} for the user,
which
%form the user's \emph{global} homophilous nodes that 
can be used to help predict her label.
%and 
%find the \emph{global} friends of the node in the graph.
Meanwhile,
since it has been pointed out in~\cite{zhu2020beyond} that
the 2-hop neighborhood of a node under some mild condition will be homophily-dominant in expectation,
some models
amplify a node's neighborhood with multi-hop neighbors.
%which aims to include more homophilous nodes.
%Despite the success,
%Some models use multi-hop neighbors to.
%Intuitively,
%the larger the receptive field,
%the more homophilous nodes will be included.
However,
how large the neighborhood size should be set for different nodes is a challenge.
Further, 
for those homophilous nodes excluded in the 
neighborhood,
they will not be utilized in information aggregation.
Therefore,
we propose to leverage 
the {global} homophily for a node in the graph by adding all the nodes to its neighborhood.
As a side effect,
in this case,
%However,
%a large node neighborhood could adversely affect the model efficiency. 
%In the extreme case,
%for an arbitrary node,
%if we add all the nodes in the graph to its neighborhood,
the traditional neighborhood aggregation will have a quadratic time complexity,
which is practically infeasible.
Further,
more heterophilous nodes will be included in the neighborhood, 
which could adversely affect the model performance.
Therefore,
a research question arises:
\emph{
%To capture more homophilous neighbors for a node,
% in graphs with heterophily,
Can we find global homophily for a node
and
%add all the nodes in the graph to its receptive field and 
develop a GNN model that is both effective and efficient for heterophilous graphs?}
%Since a large node neighborhood could adversely affect the model efficiency,
%%for each node,
%these models only consider a small set of multi-hop neighbors in a node's receptive field.
%For other homophilous nodes excluded in the receptive field,
%they are not utilized.
%but not all the nodes in the graph.
%complete employment of all the all the nodes in the graph.

%Further,
%most existing GNN models for heterophilous graphs
%are evaluated on benchmark datasets released in~\cite{pei2020geom}.
%These datasets are of small sizes, 
%which leads to the unconvincing model evaluation.
%Recently,
%\citet{lim2021large} forge new large-scale benchmarks that consist of million-scale nodes. 
%This poses new challenge for designing scalable GNN models for heterophilous graphs.

In this paper,
to find \textbf{Glo}bal homophily for nodes in graphs with heterophily,
we propose an effective and scalable \textbf{GNN} model, namely, {\ada}.
%We expand a node's neighborhood by adding all the nodes in the graph to its receptive field.
%After that,
%we calculate the similarity between the node and its neighbors.
%we inspect all the homophilous nodes to the node.
%but not limited to a restricted set of neighbors.
%Specifically,
In the $l$-th convolutional layer,
inspired by the linear subspace model~\cite{liu2012robust},
%each node is linearly characterized by all the nodes in the graph.
we linearly characterize each node by all the nodes in the graph and derive a coefficient matrix $Z^{(l)}$ 
such that $Z_{ij}^{(l)}$ describes the importance of node $x_j$ to node $x_i$.
We formulate the characterization problem as an optimization problem 
that has a closed-form solution for $Z^{(l)}$.
After that,
%with the learned matrix $Z$,
taking $Z^{(l)}$ as the weight matrix,
we generate node embedding matrix $H^{(l)}$ by 
aggregating information from global nodes.
%This process introduces global friends for a node.
%Further,
Note that
directly computing such $Z^{(l)}$ and $Z^{(l)}$-based neighborhood aggregation
lead to cubic and quadratic time complexities w.r.t. the number of nodes, respectively.
Hence, 
we avoid calculating $Z^{(l)}$ directly and reorder matrix multiplication in neighborhood aggregation,
which effectively reduces the time complexity to linear.
Finally,
we mathematically show that 
both $Z^{(l)}$ and the generated node embedding matrix $H^{(l)}$ 
have the \emph{grouping effect}~\cite{lu2012robust}, i.e., 
for any two nodes in a graph, no matter how distant they are,
if they share similar features and local structures,
their embedding vectors will be close to each other.
This helps explain the effectiveness of our models.
%which explains the effectiveness of our models. 
We summarize the main contributions of our paper as:

\noindent{\small$\bullet$}
We propose two effective and efficient GNN models, \ada\ and \ada++.
%for graphs with heterophily,
%which aim to find a node's global friends for information aggregation.
%which generate a node's embedding by aggregating information from all the nodes in the graph.
%We also present a well-designed strategy to speed up the model.
%Specifically,
%both models learn a coefficient matrix to capture node correlations, 
%based on which neighborhood aggregation is performed.
%A well-designed strategy to speed up the aggregation is also presented.

%we linearly characterize node correlations, define neighborhood aggregation and 
%further present a well-designed strategy to speed up the models.

\noindent{\small$\bullet$}
We theoretically show that both $Z^{(l)}$ and the generated node embedding matrix $H^{(l)}$
have the {grouping effect}.

\noindent{\small$\bullet$}
%We formulate the problem of linear characterization as an optimization problem that 
%has a closed-form solution.
%The derived coefficient matrix $Z$ allows signed values and acts as the attention weight matrix,
%based on which 
We combine
low-pass and high-pass convolutional filters in neighborhood aggregation,
as $Z^{(l)}$ allows signed values.
%based on which a node's embedding is generated by aggregating information from all the nodes in the graph but not a small set of multi-hop neighbors.

\noindent{\small$\bullet$}
We show the superiority of our models against 11 other methods on 15 benchmark datasets of diverse domains, sizes and graph heterophilies.
%to show the superior performance.
%In particular,
%we compare our models with 10 other methods on 15 benchmark datasets with diverse domains, sizes and graph heterophilies.
%The results show that 
%our models 
%can consistently achieve superior performance against other competitors over all the datasets;
%our methods also converge fast with high efficiency.

\section{Related Work}
\label{sec:related}
GNNs have recently received significant interest 
for the superior performance on graph-based learning.
The early model GCN~\cite{kipf2016semi} extends
the convolution operation from regular data to irregular graph-structured data.
%which avoids explicit forms of graph Laplacian regularization.
GCN is a spectral model~\cite{bruna2013spectral,defferrard2016convolutional}, 
which decomposes graph signals via graph Fourier transform and convolves on the spectral components.
There are also a class of spatial GNN models that directly aggregate information from spatially nearby neighbors of a node.
For example,
GraphSAGE~\cite{hamilton2017inductive} generates a node's embedding by aggregating information from a fixed number of neighbors.
GAT~\cite{velivckovic2017graph} introduces the attention mechanism to learn the importance of a node's neighbors
and aggregates information from these neighbors based on the learned weights.
%To better understand GNNs,
%some works~\cite{zhu2021interpreting,fu2020understanding} formulate a number of representative GNN models in a unified optimization framework.
%There also exist works~\cite{xu2018powerful,chen2019equivalence} that explore the expressive power of GNNs for graph isomorphism.
Further,
GNNs have been widely studied from various perspectives,
such as the over-smoothing problem~\cite{zhao2019pairnorm,rong2019dropedge},
%wu2019simplifying,chen2020simple}, 
the adversarial attack
and defense~\cite{dai2018adversarial,zhu2019robust},
%zugner2019adversarial,jin2020graph},
and the model explainability~\cite{vu2020pgm,shan2021reinforcement}.
%yuan2021explainability,ying2019gnnexplainer}.
%For a comprehensive survey of GNNs, see~\cite{wu2020comprehensive,zhou2020graph}.

%Despite the success,
%many GNN models implicitly assume that 
%connected nodes in the graph are more likely to share similar features and have the same label.

There are also works~\cite{zhu2020beyond,bo2021beyond,chien2020adaptive,yan2021two,suresh2021breaking,pei2020geom,dong2021graph,lim2021large,yang2021graph,luan2021heterophily,zhu2020graph,liu2021non}
that extend
GNNs to heterophilous graphs. 
%Geom-GCN~\cite{pei2020geom} is an early model that maps a graph into a continuous latent space
%with a geometric aggregation scheme.
%It then uses the geometric relationships defined in the space to build 
%structural neighborhoods for aggregation.
%For each node in the graph,
%its neighborhood includes not only adjacent neighbors, 
%but also distant nodes that share a certain similarity. 
%Geom-GCN~\cite{pei2020geom} is an early model 
%that constructs a node's neighborhood by 
%both its adjacent neighbors
%and distant nodes that share a certain similarity with it.
%After that,
Some methods propose to leverage both low-pass and high-pass convolutional filters 
in neighborhood aggregation.
For example,
%FAGCN~\cite{bo2021beyond}
%adaptively learns low-frequency and high-frequency signals from a node's neighbors.
GPR-GNN~\cite{chien2020adaptive}
%combines GNN with Generalized PageRank technique.
adapts to the homophily/heterophily structure of a graph
by learning signed weights for node embeddings in different propagation steps.
ACM-GCN~\cite{luan2021heterophily} applies both low-pass and high-pass filters for each node in a layer,
and adaptively fuses the generated node embeddings from each filter.
%while GPR-GNN~\cite{chien2020adaptive} only utilizes one type of convolutional filter.
Further,
some methods enlarge the node neighborhood size to include more homophilous nodes.
As a representative model, 
H$_2$GCN~\cite{zhu2020beyond} 
presents three designs to improve the performance of GNNs under heterophily:
ego and neighbor embedding separation, 
higher-order neighborhood utilization and 
intermediate representation combination.
%CPGNN~\cite{zhu2020graph}
%addresses the heterophily problem by incorporating 
%into GNNs a compatibility matrix that 
%models the probability of connections between nodes from different classes.
%The matrix is then used as the weight transformation matrix in each propagation layer.
%by either amplifying node receptive field or stacking convolutional layers.
%However, 
%amplifying node receptive field could adversely affect the model efficiency,
%while 
%stacking convolutional layers could easily lead to the over-smoothing problem.
%Recently,
WRGAT~\cite{suresh2021breaking}
transforms the original graph into a multi-relational one that 
contains both raw edges and newly constructed edges.
The new edges 
can connect distant nodes and are weighted by 
node local structural similarity.
%Then GNNs are performed on the multi-relational graph.
There also exist GNNs that study the graph heterophily issue from other perspectives.
For example,
to jointly study
the heterophily and over-smoothing problems,
GGCN~\cite{yan2021two} allows for signed messages to be propagated from a node's neighborhood.
On the other hand,
it adopts a degree correction mechanism to rescale node degrees and further alleviate the over-smoothing problem.
%the proposed model is experimentally inefficient (we will show the results in Sec [xxx]), 
%which hinders its application to large datasets.
%These problems further motivate us to study GNNs on heterophilous graphs. 
%There also exist some special ways to handle the heterophily issue.
%For example,
To generalize GNNs to large-scale graphs,
LINKX~\cite{lim2021large}
separately embeds node features and graph topology.
%adjacency matrix and node feature matrix.
After that,
the two embeddings are combined with MLPs to generate node embeddings.
Different from all these methods,
\ada\ performs node neighborhood aggregation from the whole set of nodes in the graph,
which takes more nodes in the same class as neighbors and thus boosts the performance of GNNs on graphs with heterophily.
%Further,
%\ada\ also incorporates the local structural information of nodes to boost the model performance.

%Despite the success,
%all these methods are evaluated on small-size datasets, 
%which overshadows their effectiveness.
%To further generalize GNNs to large-scale graphs,
%LINKX~\cite{lim2021large}
%separately embeds node features and graph topology.
%%adjacency matrix and node feature matrix.
%After that,
%the two embeddings are combined with multilayer perceptrons to generate node embedding.
%However, it is observed in~\cite{lim2021large} that 
%%While LINKX is shown to be scalable,
%LINKX performs poorly on some small benchmark datasets due to its simple model architecture.
%Different from all these methods,
%our model \ada\ can achieve superior performance with high efficiency.

\section{Preliminaries}
\label{sec:preliminary}
In this section we introduce notations and concepts used in this paper.

\noindent \textbf{[Notations]}.
We denote an undirected graph without self-loops as $\mathcal{G} = (\mathcal{V}, \mathcal{E})$,
where $\mathcal{V} = \{v_i\}_{i=1}^n$ is a set of nodes and
$\mathcal{E} \subseteq \mathcal{V} \times \mathcal{V}$ is a set of edges.
Let $A$ denote the adjacency matrix such that 
%Let $\mathrm{A}$ be an adjacency matrix such that 
$A_{ij}$ represents the weight of edge $e_{ij}$ between nodes $v_i$ and $v_j$.
%Let $\mathrm{X} \in \mathbb{R}^{n\times l}$ denotes an attribute matrix,
%where the $i$-th row denotes the feature vector of node $v_i$.
For each node $v_i$,
we use $\mathcal{N}_i$ to denote $v_i$'s neighborhood,
which is the set of nodes directly connected to $v_i$.
%we represent its degree as $d_i = \sum_{j=1}^n{A_{ij}}$.
We further construct a diagonal matrix $D$ where $D_{ii}= \sum_{j=1}^n{A_{ij}}$.
We denote the node representation matrix in the $l$-th layer as $H^{(l)}$,
where the $i$-th row is the embedding vector $h_i^{(l)}$ of node $v_i$.
For the initial node feature matrix,
we denote it as $X$.
We use $Y\in \mathbb{R}^{n\times c}$ to denote the ground-truth node label matrix,
where $c$ is the number of labels in node classification and
the $i$-th row $y_i$ is the one-hot encoding of node $v_i$'s label.
%For simplicity,
%we set $\mathrm{A}_{ij} = 1$ if $e_{ij} \in \mathcal{E}$;  $\mathrm{A}_{ij} = 0$, otherwise.

\noindent \textbf{[Homophily/Heterophily]}.
The homophily/heterophily of a graph is typically defined 
based on the similarity/dissimilarity between two connected nodes w.r.t. node features or node labels.
In this paper,
we focus on homophily/heterophily in node labels.
There have been some metrics of homophily proposed.
For example,
\emph{edge homophily}~\cite{zhu2020beyond} is defined as the fraction of edges that connect nodes with the same label.
%node homophily~\cite{pei2020geom} measures the average proportion of label consistency 
%between a node and its neighbors for all the nodes in the graph.
Further,
high homophily indicates low heterophily, and vice versa.
We thus interchangeably use these two terms in this paper.

\noindent \textbf{[GNN basics]}.
The convolution operation in GNNs is typically composed of two steps:
(1) feature propagation and aggregation:
$\hat{h}_i^{(l)} = \texttt{AGGREGATE}(h_j^{(l)}, \forall v_j \in \mathcal{N}_i)$;
(2) node embedding updating:
$h_i^{(l+1)} = \texttt{Update}(h_i^{(l)}, \hat{h}_i^{(l)})$.
One of the most widely used GNN models is  
vanilla GCN~\cite{kipf2016semi},
which adopts a renormalization trick to add a self-loop to each node in the graph.
After that,
%and transforms $A$ to $\tilde{A} = A + I$, where $I$ is the identity matrix.
the normalized affinity matrix $\hat{{A}} = \tilde{D}^{-\frac{1}{2}}\tilde{A}\tilde{D}^{-\frac{1}{2}}$,
where $\tilde{A} = A + I_n$, $\tilde{D} = D + I_n$ and $I_n \in \mathbb{R}^{n\times n}$ is the identity matrix.
Note that $\hat{{A}}$ is a low-pass filter while the corresponding Laplacian matrix $L = I_n - \hat{{A}}$ is a high-pass filter.
The output in the $(l+1)$-th layer of vanilla GCN is 
$H^{(l+1)} = \sigma(\hat{A} H^{(l)}W^{(l)})$,
where $W^{(l)}$ is a learnable weight matrix and $\sigma$ is the \texttt{Relu} function.
After $L$ layers,
$H^{L}$ is then subsequently fed into a \texttt{softmax} layer to generate label probability logits and 
a \texttt{cross-entropy} function for node classification.

%\noindent \textbf{[Semi-supervised node classification]}.
%Given a graph $\mathcal{G} =(\mathcal{V}, \mathcal{E})$ and a label set $\mathcal{C}$ with $|\mathcal{C}| = c$,
%let $\mathcal{V} = \mathcal{L} \cup \mathcal{U}$,
%where $\mathcal{L}$ is a set of labeled objects and $\mathcal{U}$ is a set of unlabeled
%ones, the node classification problem is to learn a mapping $\psi$: $\mathcal{V}  \rightarrow \mathcal{C}$ to 
%predict the labels
%of nodes in $\mathcal{U}$.

\section{Algorithm}
\label{sec:algorithm}
In this section,
we describe our models.
We first show how to capture node correlations and derive the coefficient matrix $Z$ in each layer.
After that,
we introduce how to accelerate neighborhood aggregation based on $Z$.
Further,
we
theoretically prove that both $Z$ and $H$ have the desired grouping effect.
Finally,
we summarize \ada\ and upgrade the model to \ada++. 
The overall framework of \ada\ is given in Figure~\ref{fig:framework}.

\begin{figure}[!htbp]
    \centering
        \includegraphics[width = 0.8\linewidth]{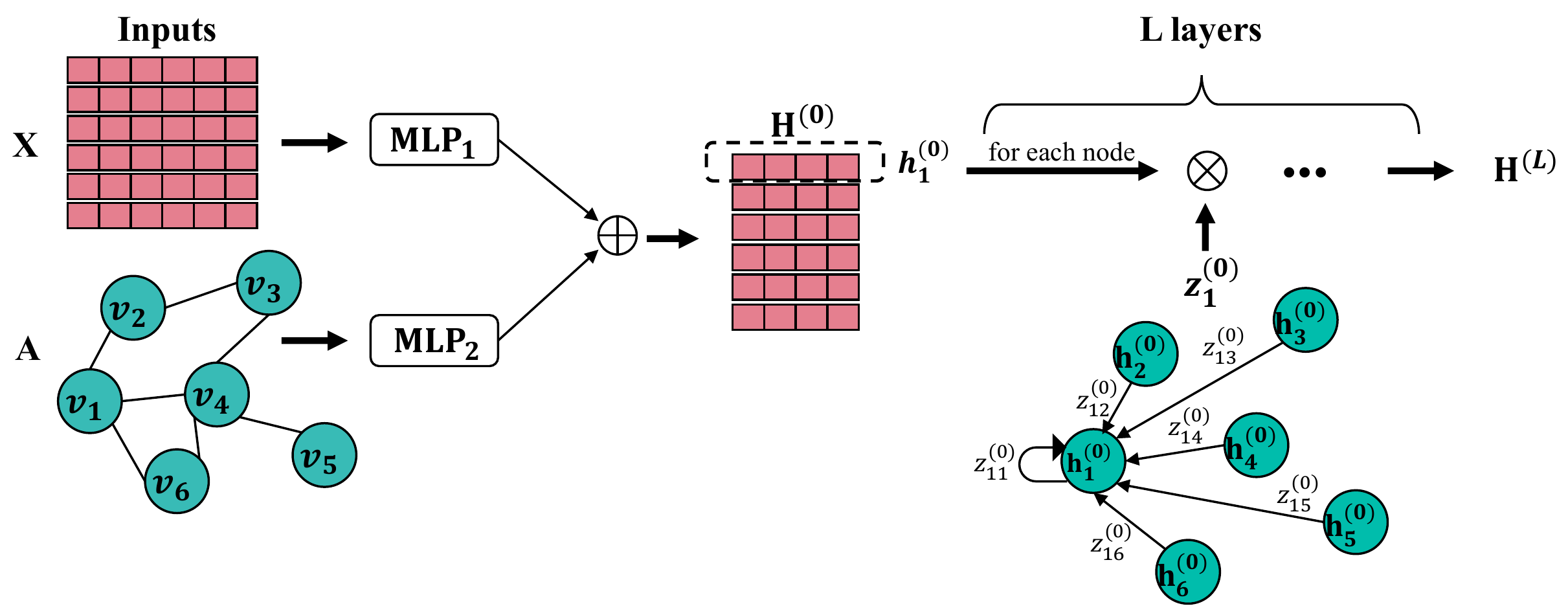}
        \caption{The overall framework of \ada. In each layer,
        we derive a coefficient matrix, based on which a node's embedding is generated by aggregating information from global nodes.}
        \label{fig:framework}
\end{figure}

\subsection{Coefficient matrix}
In graphs with heterophily,
connected nodes are more likely to have different labels
while distant nodes could be from the same class.
%To include more homophilous neighbors,
We thus have to enlarge a node's neighborhood to leverage more distant nodes.
A straightforward way is to use all the nodes in the graph.
Inspired by the linear subspace model~\cite{liu2012robust},
we characterize each node by all the nodes in the graph.
Specifically,
in the $l$-th layer,
we can determine a coefficient matrix $Z^{(l)} \in \mathbb{R}^{n\times n}$:
\begin{equation}
\label{eq:coe}
H^{(l)} = Z^{(l)}H^{(l)} + O^{(l)},
\end{equation}
where $O^{(l)}$ is a noise matrix.
One can interpret $Z_{ij}^{(l)}$ as a value that reflects how well node $x_j$ characterizes $x_i$,
so $Z^{(l)}$ plays the role of a weight matrix.
%On the other hand,
%$Z_{ij}$ is allowed to be a signed value,
%which automatically combines low-pass and high-pass filters in the convolution operation.
Note that $Z_{ij}^{(l)}$ allows signed values.
%the convolution operation over a node's neighbors is performed on all the nodes in the graph 
%but not a limited set of multi-hop neighbors.
%so neighborhood aggregation
%automatically combines the low-pass and high-pass convolutional filters.
On the one hand,
the more similar two nodes are,
the more likely that one node can be represented by the other.
Homophilous neighbors will thus be assigned large positive coefficients.
On the other hand,
heterophilous neighbors will be given small positive or negative coefficients due to the dissimilarities.
%With these learned weights,
%Based on $Z^{(l)}$,
%neighborhood aggregation will generate 
%a node's embedding that is close to that of other similar nodes and far away from dissimilar ones.
Further,
%Eq.~\ref{eq:coe} describes a neighborhood aggregation process.
since
it has been pointed out in~\cite{liu2020towards} that the over-smoothing problem
is caused by the coupling of neighborhood aggregation and feature transformation,
we decouple these two processes by 
first performing
%first use a two-layer MLP to preprocess the node feature matrix and the adjacency matrix
feature transformation to generate $H^{(0)}$.
Specifically,
we use MLPs to map feature matrix and adjacency matrix 
into $H_X^{(0)} \in \mathbb{R}^{n\times c}$ and $H_A^{(0)} \in \mathbb{R}^{n\times c}$, respectively:
\begin{equation}
\label{eq:h0xa}
H_X^{(0)} = \texttt{MLP}_1(X),\ H_A^{(0)} = \texttt{MLP}_2(A).
\end{equation}
Here, $c$ is the number of labels.
Then
we derive the initial node embedding matrix $H^{(0)}$:
\begin{equation}
\label{eq:h0}
H^{(0)} = (1-\alpha)H_X^{(0)}  + \alpha H_A^{(0)},
\end{equation}
where $\alpha \in [0,1] $ is the term weight.
%We next show the derivation of $Z$.
Note that $H^{(0)}$ captures the information of both nodes' features and connectivities. 
Inspired by~\cite{klicpera2018predict},
we use skip connection 
%enforce $H^{(l)}$ to contain some information of $H^{(0)}$
and further modify $H^{(l)}$ in Equation~\ref{eq:coe} into:
\begin{equation}
\label{eq:coe1}
H^{(l)} = (1-\gamma) Z^{(l)}H^{(l)} + \gamma H^{(0)} + O^{(l)},
\end{equation}
where $\gamma \in [0,1]$ is a hyper-parameter that balances the term importance.
Here,
we characterize node correlations based on $H^{(l)}$.
However,
as suggested in~\cite{suresh2021breaking}, if two nodes share similar local graph structures,
they are more likely to have the same label.
%While
%$Z$ is a matrix that reflects the relationships between nodes in the graph,
%the graph adjacency matrix also conveys the relationships between nodes.
%In heterophilous graphs,
%multi-hop neighbors could include more homophilous neighbors for a node,
We thus measure node correlations in terms of both feature similarity and topology similarity
by further regularizing $Z^{(l)}$ with nodes' multi-hop reachabilities.
%We
%further use nodes' multi-hop reachability to regularize $Z$.
We have the following objective function:
\begin{equation}
\label{eq:obj}
\scriptsize
\min_{Z^{(l)}} \Vert H^{(l)}-(1-\gamma)Z^{(l)}H^{(l)}-\gamma H^{(0)}\Vert_F^2  +\beta_1 \Vert Z^{(l)} \Vert_F^2 + \beta_2 \Vert Z^{(l)}-\sum_{k=1}^{K}\lambda_{k}\hat{A}^k\Vert_F^2
\end{equation}
where $\beta_1$ and $\beta_2$ are weighting factors for adjusting the importance 
of different components,
and $K$ is the maximum hop count.
To show the importance of the $k$-hop graph connectivity,
we further introduce a
learnable parameter
$\lambda_k$.
The objective function consists of three terms.
The first term 
reduces noise and
drives the linear representation for nodes to be close to their own embeddings,
the second term is a Frobenius norm, 
and the third term regularizes $Z^{(l)}$ by the multi-hop regularized graph adjacency matrices.
A closed-form solution $Z^{{(l)}*}$ to the optimization problem is 
\begin{equation}
\label{eq:zstar}
\scriptsize
\nonumber
\begin{split}
Z^{(l)*}   = &  \left [(1-\gamma)H^{(l)}(H^{(l)})^T +\beta_2 \sum_{k=1}^{K}\lambda_{k}\hat{A}^k - \gamma(1-\gamma) H^{(0)}(H^{(l)})^T\right ]  \cdot \\
& \left[(1-\gamma)^2H^{(l)}(H^{(l)})^T + (\beta_1 + \beta_2) I_n\right ]^{-1}\\
%||H^{(l)}-(1-\gamma)H^{(l)}Z-\gamma H^{(0)}||_F^2 + \alpha_1 ||Z||_F^2 + \alpha_2 ||Z-\sum_{k=0}^{K}\lambda_{k}\hat{A}^k||_F^2,
\end{split}
\end{equation}

\subsection{Aggregation acceleration}
\label{sec:acc}
Based on $Z^{(l)*}$,
we can write neighborhood aggregation in the $l$-th layer: 
\begin{equation}
\label{eq:conv}
H^{(l+1)} = (1-\gamma) Z^{(l)*} H^{(l)} + \gamma H^{(0)}.
\end{equation}
However,
directly updating $H^{(l+1)}$ by Equation~\ref{eq:conv} is infeasible
due to the cubic time complexity in computing
$Z^{(l)*}$ and the quadratic time complexity in calculating $Z^{(l)*} H^{(l)}$.
%has a term of matrix inversion that 
%are computationally expensive,
To accelerate the updates of $H^{(l+1)}$,
%To accelerate our model,
%based on the Woodbury formula, 
%After that,
%based on Eq.~\ref{eq:zstar} and Eq.~\ref{eq:inverse},
instead of directly calculating $Z^{(l)*}$,
we transform Equation~\ref{eq:conv} into (see Section~\ref{sec:agg} in the appendix):
{\begin{scriptsize}
\begin{equation}
%\nonumber
\label{eq:Hexpansion}
\begin{split}
%& H^{(l+1)} = (1-\gamma) H^{(l)}Z^* + \gamma H^{(0)} \\
%& =  \left[\frac{1-\gamma}{\beta} H^{(l)} - \frac{1-\gamma}{\beta^2} H^{(l)}(H^{(l)})^T \left[ \frac{1}{(1-\gamma)^2}I + \frac{1}{\beta} H^{(l)}(H^{(l)})^T \right]^{-1} H^{(l)}\right] \\
H^{(l+1)}  = & (1-\gamma)H^{(l)} (H^{(l)})^TQ^{(l+1)}  + \beta_2 \sum_{k=1}^{K}\lambda_{k}\hat{A}^k Q^{(l+1)}  \\
& - \gamma(1-\gamma) H^{(0)} (H^{(l)})^T  Q^{(l+1)}  + \gamma H^{(0)} \\
\end{split}
\end{equation}
\end{scriptsize}
}
%\begin{equation}
%\label{eq:Hexpansion}
%%\small
%\begin{split}
%%& H^{(l+1)} = (1-\gamma) H^{(l)}Z^* + \gamma H^{(0)} \\
%%& =  \left[\frac{1-\gamma}{\beta} H^{(l)} - \frac{1-\gamma}{\beta^2} H^{(l)}(H^{(l)})^T \left[ \frac{1}{(1-\gamma)^2}I + \frac{1}{\beta} H^{(l)}(H^{(l)})^T \right]^{-1} H^{(l)}\right] \\
%H^{(l+1)}  = & - \frac{1-\gamma}{\beta^2} H^{(l)} \left[ \frac{1}{(1-\gamma)^2}I + \frac{1}{\beta} (H^{(l)})^TH^{(l)} \right]^{-1} (H^{(l)})^T  Q^{(l+1)}   \\
%& + \frac{1-\gamma}{\beta} Q^{(l+1)} + \gamma H^{(0)}  \\
%\end{split}
%\end{equation}
where 
\begin{scriptsize}
\begin{equation}
%\nonumber
\label{eq:q}
\begin{split}
 Q^{(l+1)}  =  & \frac{1-\gamma}{\beta_1+\beta_2} H^{(l)}  - \frac{1-\gamma}{(\beta_1+\beta_2)^2} H^{(l)} \cdot \\
& \left[ \frac{1}{(1-\gamma)^2}I_c + \frac{1}{\beta_1+\beta_2} (H^{(l)})^TH^{(l)} \right]^{-1}(H^{(l)})^T H^{(l)} \\
\end{split}
\end{equation}
\end{scriptsize}
%\begin{equation}
%\label{eq:q}
%\begin{split}
%Q^{(l+1)}  = & (1-\gamma)H^{(l)} (H^{(l)})^T H^{(l)} +\beta \sum_{k=1}^{K}\lambda_{k}\hat{A}^k H^{(l)} \\
%& - \gamma(1-\gamma) H^{(0)} (H^{(l)})^T H^{(l)} \\
%\end{split}
%\end{equation}
Here,
$H^{(l)}, Q^{(l+1)} \in \mathbb{R}^{n\times c}$, $I_c \in \mathbb{R}^{c\times c}$ is the identity matrix and $c$ is the number of labels.
In this way,
we avoid computing $Z^{(l)*}$ directly and can
accelerate the updates of $H^{(l+1)}$ by matrix multiplication reordering.
Specifically,
we first compute $Q^{(l+1)}$,
where the second term is calculated from right to left.
In particular,
the matrix inversion
is
performed on a matrix in $\mathbb{R}^{c\times c}$,
whose time complexity
is only $O(c^3)$ and $c$ is generally a very small number.
This significantly improves the model efficiency. 
The overall time complexity of updating $Q^{(l+1)}$ is $O(nc^2 + c^3)$, where $c^2 \ll n$.
After $Q^{(l+1)}$ is calculated,
we  
then update $H^{(l+1)}$.
We compute each term of $H^{(l+1)}$ in a similar right-to-left manner.
%We first compute each term in $Q^{(l+1)}$ from right to left.
For example,
to calculate $H^{(l)}(H^{(l)})^TQ^{(l+1)}$,
we first compute $(H^{(l)})^TQ^{(l+1)}$ but not $H^{(l)}(H^{(l)})^T$.
This reduces the time complexity to be $O(nc^2)$, but not $O(n^2c)$.
When calculating $\sum_{k=1}^{K}\lambda_{k}\hat{A}^k Q^{(l+1)}$,
we can first compute $\hat{A} Q^{(l+1)}$ due to the sparsity of $\hat{A}$ for a general graph,
which only requires a time complexity of $O(k_1cn)$.
Here, $k_1$ is the average number of nonzero entries in a row of $\hat{A}$.
While $\hat{A}Q^{(l+1)}$ generates a dense matrix,
we can further employ the sparsity of $\hat{A}$ to get $\hat{A}^2Q^{(l+1)}$.
In this way,
we can sequentially derive $\hat{A} Q^{(l+1)}$, $\hat{A}^2 Q^{(l+1)}$, ..., $\hat{A}^K Q^{(l+1)}$ in $O(k_1cn)$.
%We compute the first term in Eq.~\ref{eq:Hexpansion} in a similar right-to-left manner.
%leading to a time complexity of $O(nc^2 + c^3)$, where $c^2 \ll n$.
%The overall time complexity of updating $H^{(l+1)}$ is $O(nc^2 + c^3)$, where $c^2 \ll n$.
In summary,
%since $c^2 \ll n$ and $c^3 \ll n$,
the total time complexity 
to update $H^{(l+1)}$ by Equation~\ref{eq:Hexpansion}
is $O(k_2n)$, where $k_2$ is a coefficient and $k_2 \ll n$.

\subsection{Grouping effect}
\label{sec:ge}
For a node $v_i$,
we denote $\hat{a}_i^k$ as
the $i$-th row of $\hat{A}^k$,
which represents $v_i$'s $k$-hop node reachability in a graph.
Given two nodes $v_i$ and $v_j$,
if they have similar feature vectors and local graph structures,
their characterizations from other nodes are expected to be similar.
%This has been defined as the grouping effect~\cite{lu2012robust,li2018rosc}.
Formally, we have
%Previous works~\cite{lu2012robust,li2018rosc} have shown that if $Z$ has \emph{grouping effect},
%given two highly correlated nodes $v_i$ and $v_j$,
%their characterizations of other objects are similar.

\begin{definition}
\textbf{(Grouping effect~\cite{li2020cast})}.
Given a set of nodes $\mathcal{V} = \{v_i\}_{i=1}^n$,
%let $\{h_i^{(l)}\}_{i=1}^n$ denote the node embedding in the $l$-th layer.
let $v_i \rightarrow v_j$ denote the condition that 
(1) $\lVert x_i - x_j \rVert_2 \rightarrow 0$
and (2) $\lVert \hat{a}_i^k - \hat{a}_j^k\rVert_2 \rightarrow 0,\ \forall k \in [1,K]$.
A matrix $Z$ is said to have grouping effect if 
\begin{equation}
v_i \rightarrow v_j \Rightarrow |Z_{ip} - Z_{jp}| \rightarrow 0, \forall 1 \leq p \leq n.
\end{equation}
\end{definition}

We next theoretically show the grouping effect of $Z^{(l)*}$, $(Z^{(l)*})^T$ and $H^{(l+1)}$
%In the following discussion,
%we use 
%$z_i^*$ to denote the $i$-th row of $Z^*$,
%which is the coefficient vector for representing node $v_i$.
by giving the following lemmas.
Proofs of all these lemmas
%~\ref{lemma2}, Lemma~\ref{lemma3} and Lemma~\ref{lemma:z-star}
are deferred to Section~\ref{sec:proof} in the appendix.

\begin{lemma}
\label{lemma2}
$\forall 1 \leq i, j, p \leq n$,
\begin{scriptsize}
\begin{equation}
\label{eq:zip}
\nonumber
\begin{split}
|Z_{ip}^{(l)*}-Z_{jp}^{(l)*}| & \leq \frac{1-\gamma}{\beta_1 + \beta_2}\Vert {h}_i^{(l)} - {h}_j^{(l)} \Vert_2 \Vert (h_p^{(l)})^T - (1-\gamma)^2(H^{(l)})^T R \Vert_2   \\
 & +\frac{\gamma(1-\gamma)}{\beta_1 + \beta_2} \Vert h_i^{(0)} - h_j^{(0)}\Vert_2 \Vert  (h_p^{(l)})^T - (1-\gamma)^2(H^{(l)})^T R   \Vert_2 \\
 &+\frac{\beta_2(1-\gamma)^2}{\beta_1 + \beta_2} \sum_{k=1}^{K}\lambda_{k}\Vert \hat{a}_i^k - \hat{a}_j^k\Vert_2 \Vert R \Vert_2 \\
 & + \frac{\beta_2}{\beta_1+\beta_2} \sum_{k=1}^{K}\lambda_{k}|\hat{A}_{{ip}}^k - \hat{A}_{{jp}}^k|  \\
\end{split}
\end{equation}
\end{scriptsize}
where {\small $R = \left [(1-\gamma)^2H^{(l)}(H^{(l)})^T + (\beta_1+\beta_2) I_n\right ]^{-1} H^{(l)} (h_p^{(l)})^T $}.
\end{lemma}

\begin{lemma}
\label{lemma3}
$\forall 1 \leq i, j, p \leq n$,
\begin{scriptsize}
\begin{equation}
\label{eq:norm}
\nonumber
|Z_{pi}^{(l)*}-Z_{pj}^{(l)*}| \leq  \frac{\eta (1-\gamma) \lVert h_i^{(l)} - h_j^{(l)}\rVert_2
 + \beta_2 \sum_{k=1}^{K}\lambda_{k}|\hat{A}_{{pi}}^k - \hat{A}_{{pj}}^k|}{\beta_1 + \beta_2}
\end{equation}
\end{scriptsize}
where {\small $\eta = \sqrt{\lVert{h}_p^{(l)}-\gamma h_p^{(0)}\rVert_2^2 + \beta_2 \lVert\sum_{k=1}^{K}\lambda_{k}\hat{a}_{p}^k\rVert_2^2}$}.
\end{lemma}

\begin{lemma}
\label{lemma:z-star}
Matrices $Z^{(l)*}$, $(Z^{(l)*})^T$ and $H^{(l+1)}$ all have grouping effect.
\end{lemma}

%\begin{lemma}
%Matrix $H^{(l+1)}$ has grouping effect.
%\end{lemma}
%\begin{proof}
%From Eq.~\ref{eq:conv},
%$H^{(l+1)}$ is updated based on $H^{(0)}$ and $H^{(l)}Z^*$.
%We next show both $H^{(0)}$ and $H^{(l)}Z^*$ have grouping effect.
%If $v_i \rightarrow v_j$,
%we can easily have $||x_i - x_j||_2 \rightarrow 0$
%and $||a_i - a_j||_2 \rightarrow 0$.
%Then based on Equations~\ref{eq:h0xa} and~\ref{eq:h0},
%we can induce that $H^{(0)}$ has grouping effect.
%After that,
%since $Z^*$ has grouping effect,
%the linear representations $H^{(l)}Z^*$ also has grouping effect.
%Therefore,
%we can conclude that 
%$H^{(l+1)}$ has grouping effect.
%\end{proof}

The grouping effect of $Z^{(l)*}$, $(Z^{(l)*})^T$ and $H^{(l+1)}$
indeed explains the effectiveness of our model.
In fact, 
for any two nodes $v_i$ and $v_j$,
no matter how distant they are in a graph,
if they share similar feature vectors and local structures,
we conclude that
(1) they will be given similar coefficient vectors;
(2) they will play similar roles in characterizing other nodes;
and (3) they
will be given similar representation vectors.
On the other hand,
%Eq.~\ref{eq:conv}
%to generate a node's embedding,
%we utilize all the nodes in the graph as its neighborhood.
in graphs with heterophily,
adjacent nodes are more likely to be dissimilar
and they will thus be given different embeddings.
%distant nodes with high similarity in terms of features and reachabilities are assigned large weights to represent each other.
%the neighborhood aggregation over all the nodes in a graph but not limited to its adjacent neighbors.
%Further,
%since we consider node similarity in terms of both features and reachabilities,
%for those that share low feature similarity but high reachability similarity,
%using one to characterize the other will lead to a small coefficient.
Further,
for two nodes with low feature similarity,
using one to characterize the other can be enhanced by the regularization term of local graph structure,
if they share high structural similarity.
This also applies to nodes that have high feature similarity but low structural similarity.
After $L$ convolutional layers,
we derive $H^{(L)}$.
We then normalize  $H^{(L)}$ by a \texttt{Softmax} layer,
whose results are further fed into the \texttt{Cross-entropy} function for classification.
%In the following, 
%we will see the effectiveness of \ada\ on diverse datasets in handling graph heterophily.
%We will see 
Finally,
we summarize the pseudocodes of \ada\ in Algorithm~\ref{alg} (Section~\ref{sec:sup} of the appendix).

\subsection{\ada++}
Given $H^{l}$,
the coefficient matrix $Z^{(l)*}$ in Equation~\ref{eq:conv} plays the role of ``longitudinal'' attention
that characterizes the importance of a node to another.
In neighborhood aggregation,
not only varies the importance of a node's neighbors, but also that of hidden features.
For example,
in node classification,
the imbalance of node labels could lead to the various importance of hidden features corresponding to different labels.
Therefore,
we further upgrade our model by considering ``horizontal'' attention w.r.t. hidden features.
We introduce a diagonal matrix $\Sigma \in \mathbb{R}^{c\times c}$ such that 
$\Sigma_{ii}$ describes the importance of the $i$-th dimension in $H^{(l)}$.
We modify Equation~\ref{eq:obj} into:
\begin{equation}
\scriptsize
\label{eq:obj_new}
\min_{Z^{(l)}} ||H^{(l)}-(1-\gamma)Z^{(l)}H^{(l)}\Sigma-\gamma H^{(0)}||_F^2  +\beta_1 \Vert Z^{(l)} \Vert_F^2 + \beta_2 ||Z^{(l)}-\sum_{k=1}^{K}\lambda_{k}\hat{A}^k||_F^2,
\end{equation}
and derive the optimal solution $Z^{(l)*}$:
\begin{equation}
\label{eq:zstar_new}
\scriptsize
\begin{split}
\nonumber
Z^{(l)*}   = &  \left [(1-\gamma)H^{(l)}\Sigma (H^{(l)})^T +\beta_2 \sum_{k=1}^{K}\lambda_{k}\hat{A}^k - \gamma(1-\gamma) H^{(0)}\Sigma (H^{(l)})^T\right ]  \cdot \\
& \left[(1-\gamma)^2H^{(l)}\Sigma \Sigma(H^{(l)})^T + (\beta_1 + \beta_2) I_n\right ]^{-1}\\
%||H^{(l)}-(1-\gamma)H^{(l)}Z-\gamma H^{(0)}||_F^2 + \alpha_1 ||Z||_F^2 + \alpha_2 ||Z-\sum_{k=0}^{K}\lambda_{k}\hat{A}^k||_F^2,
\end{split}
\end{equation}
Following the same procedure in Sec.~\ref{sec:acc} and~\ref{sec:ge},
we can also accelerate neighborhood aggregation and further prove that such $Z^{(l)*}$, $(Z^{(l)*})^T$ and $H^{(l+1)}$ have the desired grouping effect.
We omit the details due to the space limitation.

\subsection{Discussion}
%\noindent \textbf{[Discussion]}.
We next discuss 
the major differences between our models and the adapted GAT model that takes global nodes as a node's neighbors.
First,
the attention weights in GAT are automatically learned and lack of interpretability, 
but $Z^{(l)}$ in our models is derived from a well-designed optimization problem and has a closed-form solution.
Second,
the attention weights in GAT are always non-negative values while $Z^{(l)}$ in our methods allows signed values. 
Therefore,
GAT only employs low-pass convolutional filters while our methods combine both low-pass and high-pass filters.
Third,
for each node,
the neighborhood aggregation performed by GAT over all the nodes in the graph
is computationally expensive, 
which has a quadratic time complexity w.r.t. the number of nodes.
However,
our methods
accelerate the aggregation and derive a linear time complexity.

\section{Experiments}
\label{sec:exp}
In this section,
we comprehensively evaluate the performance of \ada\ and \ada++.
%where the largest one contains million nodes.
In particular,
we compare them with 11 other methods on 15 benchmark datasets,
%where the largest one contains million-scale nodes.
to show the effectiveness and efficiency of our models.
Due to the space limitation,
we move experimental setup (Sec.~\ref{sec:ab}) and ablation study (Sec.~\ref{sec:setup}) to the appendix.

\subsection{Datasets}
For fairness,
we conduct experiments on 15 benchmark datasets,
which include 9 small-scale datasets released by~\cite{pei2020geom} and 6 large-scale datasets from~\cite{lim2021large}.
We use the same training/validation/test splits as provided by the original papers.
In particular,
these datasets
span various domains, scales and graph heterophilies.
The statistics of these datasets are summarized in Tables~\ref{tab:result_small} and~\ref{tab:result_large}.
Details on these datasets can be found in Section~\ref{sec:datasets} of the appendix.

\begin{table*}[!htbp]
\centering
%\resizebox{0.87\linewidth}{!}{
\caption{The classification accuracy (\%) over the methods on 9 small-scale datasets released in~\cite{pei2020geom}.
The error bar ($\pm$) denotes the standard deviation score of results over 10 trials. 
We highlight the best score on each dataset in bold and the runner-up score with underline.
Note that Edge Hom.~\cite{zhu2020beyond} is defined as the fraction of edges that connect nodes with the same label.}
\label{tab:result_small}
\resizebox{\linewidth}{!}
{
\begin{tabular}{|c|c|c|c|c|c|c|c|c|c||c|c|}
\hline
                    & \textbf{Texas} & \textbf{Wisconsin}     & \textbf{Cornell} & \textbf{Actor} & \textbf{Squirrel} & \textbf{Chameleon} & \textbf{Cora} & \textbf{Citeseer} & \textbf{Pubmed} &  \parbox[t]{2mm}{\multirow{5}{*}{\rotatebox[origin=c]{90}{\textbf{Avg. Rank}}}} \\
\textbf{ Edge Hom.}  &   0.11    &      0.21     &  0.30  &  0.22  & 0.22  &  0.23 &  0.81 &  0.74 &  0.80 &   \\
\textbf{\#Nodes}  &   183    &   251        &  183  &  7,600  & 5,201 & 2,277  & 2,708 & 3,327 &  19,717 &  \\
 \textbf{\#Edges}  &    295   &    466       &  280  &  26,752  & 198,493 & 31,421  &  5,278 & 4,676  &  44,327 &   \\
 \textbf{\#Features}  &   1,703    &    1,703       &  1,703   &  931  &  2,089 &  2,325 &  1,433 & 3,703  & 500  &   \\
 \textbf{\#Classes} &    5   &     5      & 5  &  5  & 5 & 5  & 6 & 7  &  3  & \\ \hline   
     MLP                   &  $80.81 \pm 4.75$  &     $85.29\pm 3.31$        &  $81.89 \pm 6.40 $  &  $36.53 \pm 0.70$  &  $28.77 \pm 1.56$ & $46.21 \pm 2.99 $  & $75.69 \pm 2.00 $  & $74.02 \pm 1.90$  & $87.16\pm 0.37$    & 9.72 \\
     GCN                    &  $55.14 \pm 5.16$    &  $51.76 \pm 3.06$          & $60.54 \pm 5.30$ & $27.32\pm 1.10$ &   $53.43\pm 2.01$  & $64.82 \pm 2.24$ & $86.98 \pm 1.27$ &  $76.50 \pm 1.36$    &  $88.42 \pm 0.50$  & 10.22 \\
     GAT                    &  $52.16 \pm 6.63$   &   $49.41 \pm 4.09$          &$61.89 \pm 5.05$ & $27.44 \pm 0.89$ &  $40.72 \pm 1.55$ &  $60.26 \pm 2.50$  & $ 87.30 \pm 1.10$ & $76.55 \pm 1.23$ & $86.33 \pm 0.48 $  &  11.11  \\    
      MixHop               &   $77.84 \pm 7.73$   &     $ 75.88\pm 4.90$       & $73.51 \pm 6.34$ & $ 32.22\pm 2.34$ & $43.80\pm 1.48$ & $ 60.50\pm 2.53$  & $ 87.61\pm 0.85$ & $ 76.26\pm 1.33$ & $ 85.31\pm 0.61$  &  10.11   \\  
%     GraphSAGE     & $82.43 \pm 6.14$  &  $81.18 \pm 5.56$           & $75.95\pm 5.01$& $34.23\pm 0.99$ &$41.61\pm 0.74 $  &  $58.73\pm 1.68$  & $86.90\pm 1.04$ &  $76.04 \pm 1.30$ & $88.45 \pm 0.50$ &    \\
%     PairNorm                    &   $60.27\pm 4.34$   &   $48.43 \pm 6.14$         & $58.92 \pm 3.15$ & $27.40 \pm 1.24$ & $50.44\pm 2.04$ & $62.74 \pm 2.82$  & $85.79 \pm 1.01$ & $73.59 \pm 1.47$ &  $87.53 \pm 0.44$  &   \\
 %    Geom-GCN                  &  $66.76 \pm 2.72$   &    $64.51 \pm 3.66$         &$60.54\pm 3.67$ & $31.59\pm 1.15$ & $38.15\pm 0.92$ &  $60.00 \pm 2.81$  & $85.35 \pm 1.57$ & $\bm{78.02 \pm 1.15}$ &  $\underline{89.95 \pm 0.47}$  & \\   
   %  FAGCN                &  $ \pm $   &    $ \pm $         &$\pm $ & $\pm $ & $\pm $ &  $ \pm $  & $ \pm $ & $ \pm $ &  $\pm $  & \\     
     GCN\rom{2}                 & $77.57 \pm 3.83$  &   $80.39 \pm 3.40$          & $77.86 \pm 3.79$ &  $37.44 \pm 1.30$  & $38.47 \pm 1.58$ & $63.86 \pm 3.04$  & $\bm{88.37 \pm 1.25}$ & $\underline{77.33 \pm 1.48}$ &  $\bm{90.15 \pm 0.43}$   & 5.89 \\
    H$_2$GCN               &  ${84.86 \pm 7.23}$    &    ${87.65 \pm 4.98}$        & $82.70 \pm 5.28$ & $35.70 \pm 1.00$ & $36.48 \pm 1.86$ & $60.11\pm 2.15$  & $87.87\pm 1.20$ & $77.11\pm 1.57$  & $89.49 \pm 0.38$    &  6.72 \\ 
    WRGAT               &  $ 83.62 \pm 5.50 $    &    ${ 86.98 \pm 3.78}$        & $ 81.62 \pm 3.90$ & $ 36.53 \pm 0.77 $ & $ 48.85 \pm 0.78 $ & $65.24 \pm 0.87$  & $88.20 \pm 2.26$ & $76.81 \pm 1.89 $  & $ 88.52 \pm 0.92$    & 6.17  \\ 
    GPR-GNN              &   $78.38\pm 4.36$   &     $82.94 \pm 4.21$       & $80.27 \pm 8.11$ & $34.63 \pm 1.22$ & $31.61 \pm 1.24$ & $46.58\pm 1.71$  & $87.95 \pm 1.18$ & $77.13 \pm 1.67$ & $87.54 \pm 0.38$  &  8.83  \\ 
    GGCN              &   $\underline{84.86 \pm 4.55}$   &     $86.86\pm 3.29$       & $\underline{85.68 \pm 6.63}$ & $ \underline{37.54 \pm 1.56}$ & $55.17\pm 1.58$ & $\underline{71.14 \pm 1.84}$  & $87.95 \pm 1.05$ & $77.14 \pm 1.45$ & $89.15 \pm 0.37$  &  3.89  \\ 
       ACM-GCN               &  $ \bm{87.84  \pm 4.40}$   &     $  \bm{88.43 \pm 3.22}$       & $ 85.14 \pm 6.07 $ & $ 36.28 \pm 1.09 $ & $ 54.40  \pm 1.88 $ & $ 66.93 \pm 1.85 $  & $ 87.91 \pm 0.95 $ & $ 77.32  \pm 1.70 $ & $ \underline{90.00 \pm 0.52}$  &  3.78   \\  
       %ACM-GCN               &  $   \pm $   &     $   \pm  $       & $  \pm  $ & $  \pm  $ & $   \pm  $ & $  \pm  $  & $  \pm  $ & $   \pm  $ & $  \pm $  &    \\  
       LINKX               &   $ 74.60 \pm 8.37$   &     $75.49  \pm 5.72$       & $ 77.84 \pm 5.81$ & $ 36.10\pm 1.55$ & $\bm{61.81  \pm 1.80}$ & $ 68.42 \pm 1.38$  & $ 84.64 \pm 1.13$ & $ 73.19 \pm 0.99 $ & $ 87.86 \pm $ 0.77 &  8.78  \\  \hline
    \ada               &   $ 84.32 \pm 4.15$   &     $87.06 \pm 3.53$       & $ 83.51 \pm 4.26$ & $ 37.35\pm 1.30$ & $ 57.54 \pm 1.39$ & $ {69.78 \pm 2.42}$  & $ 88.31 \pm 1.13$ & $ \bm{77.41 \pm 1.65}$ & $ {89.62 \pm 0.35}$  &  $\underline{3.22}$  \\  
      \ada++               &   $ 84.05  \pm 4.90$   &     $\underline{88.04 \pm 3.22}$       & $ \bm{85.95 \pm 5.10}$ & $ \bm{37.70 \pm 1.40}$ & $\underline{57.88 \pm 1.76}$ & $ \bm{71.21 \pm 1.84}$  & $ \underline{88.33 \pm 1.09}$ & $ 77.22 \pm 1.78$ & $ 89.24\pm 0.39$  & $\bm{2.56}$   \\  \hline

\end{tabular}
}
\end{table*}

\begin{table*}[h]
\centering
%\resizebox{0.87\linewidth}{!}{
\caption{The classification results (\%) over the methods on 6 large-scale datasets released in~\cite{lim2021large}.
Note that 
we compare the AUC score on 
genius as in~\cite{lim2021large}.
For other datasets, 
we show the classification accuracy.
The error bar ($\pm$) denotes the standard deviation score of results over 5 trials. 
We highlight the best score on each dataset in bold and the runner-up score with underline.
Note that OOM refers to the out-of-memory error.}
\label{tab:result_large}
\resizebox{0.7\linewidth}{!}
{
\begin{tabular}{|c|c|c|c|c|c|c||c|}
\hline
                    & \textbf{Penn94} & \textbf{pokec}     & \textbf{arXiv-year} & \textbf{snap-patents} & \textbf{genius} & \textbf{twitch-gamers} &  \parbox[t]{2mm}{\multirow{5}{*}{\rotatebox[origin=c]{90}{\textbf{Avg. Rank}}}} \\
\textbf{ Edge Hom.}  &  0.47     &    0.44       &  0.22  & 0.07   & 0.61  &  0.54 &  \\
\textbf{\#Nodes}  &   41,554    &      1,632,803     & 169,343   &  2,923,922 &  421,961 & 168,114  &  \\
 \textbf{\#Edges}  &  1,362,229     &    30,622,564       &  1,166,243  &  13,975,788  & 984,979  &  6,797,557 &   \\
  \textbf{\#Features}  & 5     &     65      &    128 &  269   & 12  &  7 &   \\
 \textbf{\#Classes} &    2   &     2      & 5  &  5  & 2 & 2  & \\ \hline   
     MLP                   &  $ 73.61 \pm 0.40 $  &     $ 62.37 \pm 0.02$        &  $ 36.70 \pm 0.21 $  &  $ 31.34 \pm 0.05 $  &  $ 86.68 \pm 0.09 $ & $ 60.92 \pm 0.07 $  & 10.00\\
     GCN                     &  $ 82.47 \pm 0.27 $    &  $ 75.45 \pm 0.17 $          & $ 46.02 \pm 0.26 $ & $45.65 \pm 0.04 $ &   $87.42 \pm 0.37 $  & $ 62.18 \pm 0.26 $ & 7.00\\
     GAT                    &  $ 81.53  \pm 0.55 $   &   $71.77  \pm 6.18 $          &$ 46.05 \pm 0.51 $ & $ 45.37 \pm 0.44 $ &  $ 55.80 \pm 0.87 $ &  $ 59.89 \pm 4.12$  &  8.50 \\    
      MixHop                &   $ 83.47 \pm 0.71 $   &     $ 81.07 \pm 0.16 $       & $ 51.81 \pm 0.17$ & $ 52.16 \pm 0.09 $ & $90.58 \pm 0.16 $ & $ 65.64 \pm 0.27$  &  4.17 \\  
%     Geom-GCN                  &  $ \pm $   &    $ \pm $         &$\pm $ & $\pm $ & $\pm $ &  $ \pm $  & \\     
     GCN\rom{2}                  & $ 82.92 \pm 0.59 $  &   $78.94  \pm 0.11$          & $ 47.21\pm 0.28$ &  $ 37.88 \pm 0.69$  & $90.24 \pm 0.09$ & $ 63.39 \pm 0.61$  & 6.00 \\
   H$_2$GCN               &   $81.31\pm 0.60$    &    OOM       & $49.09 \pm 0.10$  &  OOM  &  OOM  & OOM   &  10.50 \\ 
   WRGAT               &     $74.32\pm0.53$   &   OOM          &  OOM &  OOM  & OOM  & OOM  &  11.92  \\ 
    GPR-GNN              &   $81.38 \pm 0.16 $   &     $ 78.83 \pm 0.05 $       & $ 45.07 \pm 0.21$ & $ 40.19 \pm 0.03 $ & $ 90.05 \pm 0.31$ & $61.89 \pm 0.29$  &  7.83 \\ 
    GGCN               &    OOM    &     OOM        & OOM &   OOM & OOM & OOM  &   12.25  \\ 
    ACM-GCN               &    $82.52\pm 0.96$    &    $63.81\pm5.20$         & $47.37 \pm 0.59$ &  $55.14\pm0.16$  &  $80.33 \pm 3.91$ &  $62.01\pm 0.73$&   6.83 \\ 
       LINKX                &   $ 84.71 \pm 0.52 $   &     $ 82.04  \pm 0.07 $       & $ \bm{56.00 \pm 1.34}$ & $ 61.95 \pm 0.12$ & $  \underline{90.77 \pm 0.27} $ & $66.06 \pm 0.19$  &   2.50  \\  \hline
     \ada               &   $ \underline{85.57  \pm 0.35}$   &  $\underline{83.00 \pm 0.10}$       & $ \underline{54.68 \pm 0.34}$ & $ \bm{62.09 \pm 0.27}$ & $ 90.66 \pm 0.11$ & $ \underline{66.19  \pm 0.29}$  & 2.17  \\  
      \ada++               &   $  \bm{85.74  \pm 0.42} $   &  $  \bm{83.05\pm 0.07}$       & $ 54.79  \pm 0.25$ & $  \underline{62.03 \pm 0.21} $ & $ \bm{90.91 \pm 0.13}$ & $ \bm{66.34 \pm 0.29}$  &  1.33   \\  \hline

\end{tabular}
}
\end{table*}

\subsection{Algorithms for comparison}
We compare \ada\ and \ada++ with 11 other baselines,
including
%which can be grouped into three categories:
(1) MLP;
(2) general GNN methods: 
GCN~\cite{kipf2016semi}, GAT~\cite{velivckovic2017graph},
MixHop~\cite{abu2019mixhop} and GCN\rom{2}~\cite{chen2020simple};
(3) heterophilous-graph-oriented methods:
H$_2$GCN~\cite{zhu2020beyond}, 
WRGAT~\cite{suresh2021breaking},
GPR-GNN~\cite{chien2020adaptive}, GGCN~\cite{yan2021two},
ACM-GCN~\cite{luan2021heterophily} and LINKX~\cite{lim2021large}.
For these methods specially designed for heterophilous graphs,
LINKX is a MLP-based method while others are GNN models.
Further,
although several ACM-variants are proposed in~\cite{luan2021heterophily},
ACM-GCN is reported to achieve the overall best performance on the same splits of benchmark datasets from~\cite{pei2020geom},
so we choose it as the baseline.
For other models like Geom-GCN~\cite{pei2020geom} and FAGCN~\cite{bo2021beyond},
since they have been shown to be outperformed by the state-of-the-arts,
we exclude their results in our paper.

\subsection{Performance results}
Tables~\ref{tab:result_small} and~\ref{tab:result_large} summarize the performance results of all the methods on 15 benchmark datasets.
Note that 
we compare the AUC score on 
genius as in~\cite{lim2021large}.
For other datasets, 
we show the results of classification accuracy.
Each column in the tables corresponds to one dataset.
For each dataset,
we highlight the winner's score in bold and the runner-up's with underline.
From the tables,
we make the following observations:

(1) 
MLP that uses only node features performs surprisingly well on some datasets with large heterophily, such as Actor. 
This
%linked nodes are more likely to have different labels in these networks.
shows the importance of node features for node classification in heterophilous graphs. 

(2)
Compared with the plain GNN models GCN and GAT,
MixHop and GCN\rom{2} generally perform better.
For example,
on Wisconsin, the accuracy scores of MixHop and GCN\rom{2} are $0.7588$ and $0.8039$, respectively,
which significantly outperform that of GCN and GAT. 
On the one hand,
MixHop 
amplifies a node's neighborhood with its multi-hop neighbors.
This introduces more homophilous neighbors for the node.
On the other hand,
the initial residual and identity mapping mechanisms in GCN\rom{2}
implicitly combine intermediate node representations to boost the model performance.

%(3) GNNs that are specially designed for heterophilous graphs 
%%achieve an overall better performance than other methods.
%beat other methods on 13 out of 15 datasets.
%This shows the necessity to take into account graph heterophily when designing GNN models for heterophilous graphs.

(3) Although H$_2$GCN, WRGAT and GGCN can achieve good performance on small-scale datasets,
they fail to run on very large-scale datasets due to the out-of-memory (OOM) error.
This hinders the wide application of these models.
For ACM-GCN and LINKX,
they cannot consistently provide superior results. 
For example,
LINKX ranks third on large-scale datasets,
but performs poorly on small-scale ones (rank 8th).
ACM-GCN is the winner on Texas, but its accuracy score on pokec is only $0.6381$ (the best result is $0.8305$).
While GPR-GNN leverages both low-pass and high-pass filters,
it only utilizes one type of convolutional filters in each layer,
%it assigns the same weight to all the $x$-hop neighbors ($x = 1, 2, ...$) for all the nodes,
%does not learn the importance of different neighbors of a node,
which restricts its effectiveness.

(4) \ada++ achieves the first average rank over all the datasets while \ada\ is the runner-up. 
This shows that 
both of them can consistently provide superior results on datasets in a wide range of diversity.
On the one hand,
both methods learn to utilize more neighbors with homophily for node neighborhood aggregation.
This
boosts the model performance.
On the other hand,
\ada++
further 
learns the importance of hidden features of nodes,
which improves the classification accuracy.

\subsection{Efficiency study}
In this section,
we study \ada's efficiency.
For fairness,
we compare the training time 
%for GPR-GNN, H$2$GCN, ACM-GCN, LINKX, \ada\ and \ada++,
for methods that are specially designed for graphs with heterophily.
In particular,
we make the comparison on the large-scale datasets for better efficiency illustration.
For all these methods,
we use the same training set on each dataset and run the experiments for 500 epochs.
We repeat the experiments three times and show the average training time of these methods
w.r.t. accuracy/AUC scores on the validation set in Figure~\ref{figure:runtime}.
Note that
due to the OOM error,
GGCN fails to run on these datasets and
we exclude it for comparison.
We also drop WRGAT because it takes long time to precompute the multi-relational graph.
%computes multi-relational graph as a pre-processing step
For example,
with default hyper-parameter settings,
WRGAT takes around $3200$ seconds to compute the multi-relational graph of Penn94 on a server with 48 CPUs.
However,
%during this period,
the runs of 
all other models that perform on the original graph are finished within the period.
We next recap the major difference of these methods:
all the methods except LINKX are GNN models.
Specifically,
for each node,
ACM-GCN and GPR-GNN perform convolution directly from its adjacent neighbors
while 
H$_2$GCN, \ada\ and \ada++ amplify the neighborhood of the node. 
Here,
H$_2$GCN 
considers multi-hop neighbors in the node's neighborhood while
both \ada\ and \ada++ employ the whole set of nodes in the graph.
Compared with \ada,
\ada++ further incorporates an attention mechanism 
to learn the importance of node hidden features.
Finally,
LINKX is a simple MLP-based model that does not include the graph convolution operation.

\begin{figure}[!htbp]
    \centering
    \subfigure[arXiv-year]{\includegraphics[width=0.2\textwidth]{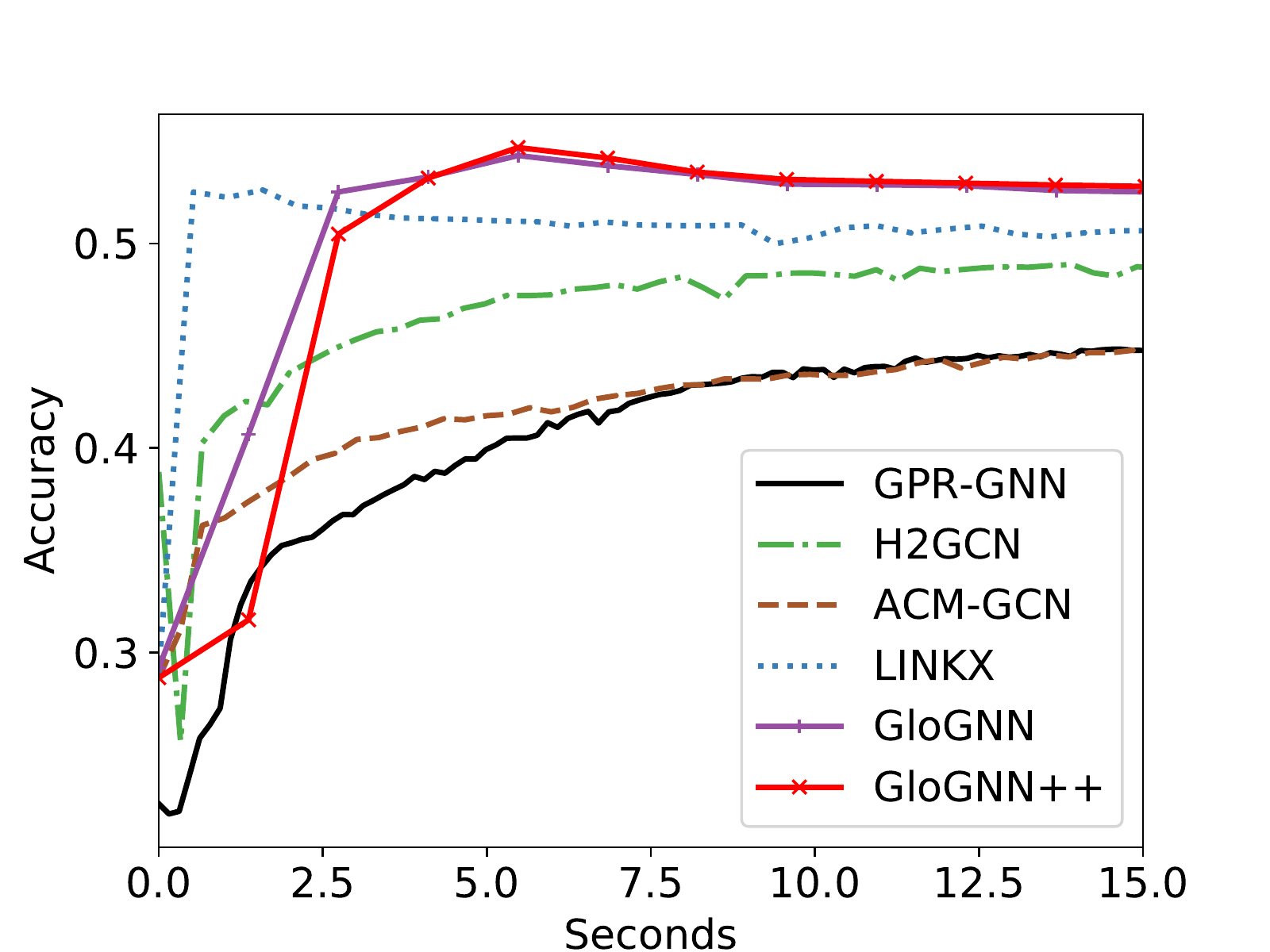}}    
    \subfigure[Penn94]{\includegraphics[width=0.2\textwidth]{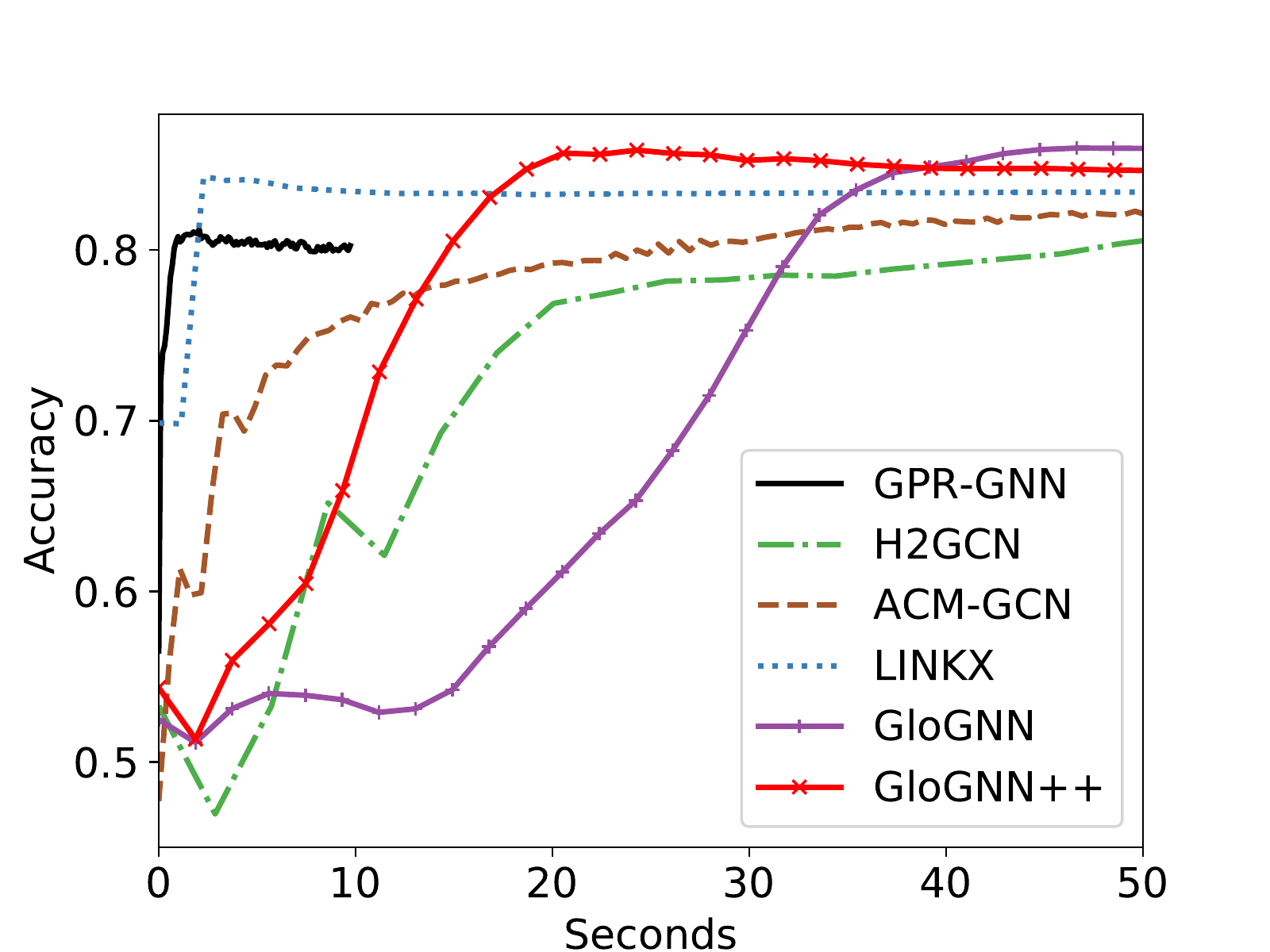}} 
    \subfigure[genius]{\includegraphics[width=0.2\textwidth]{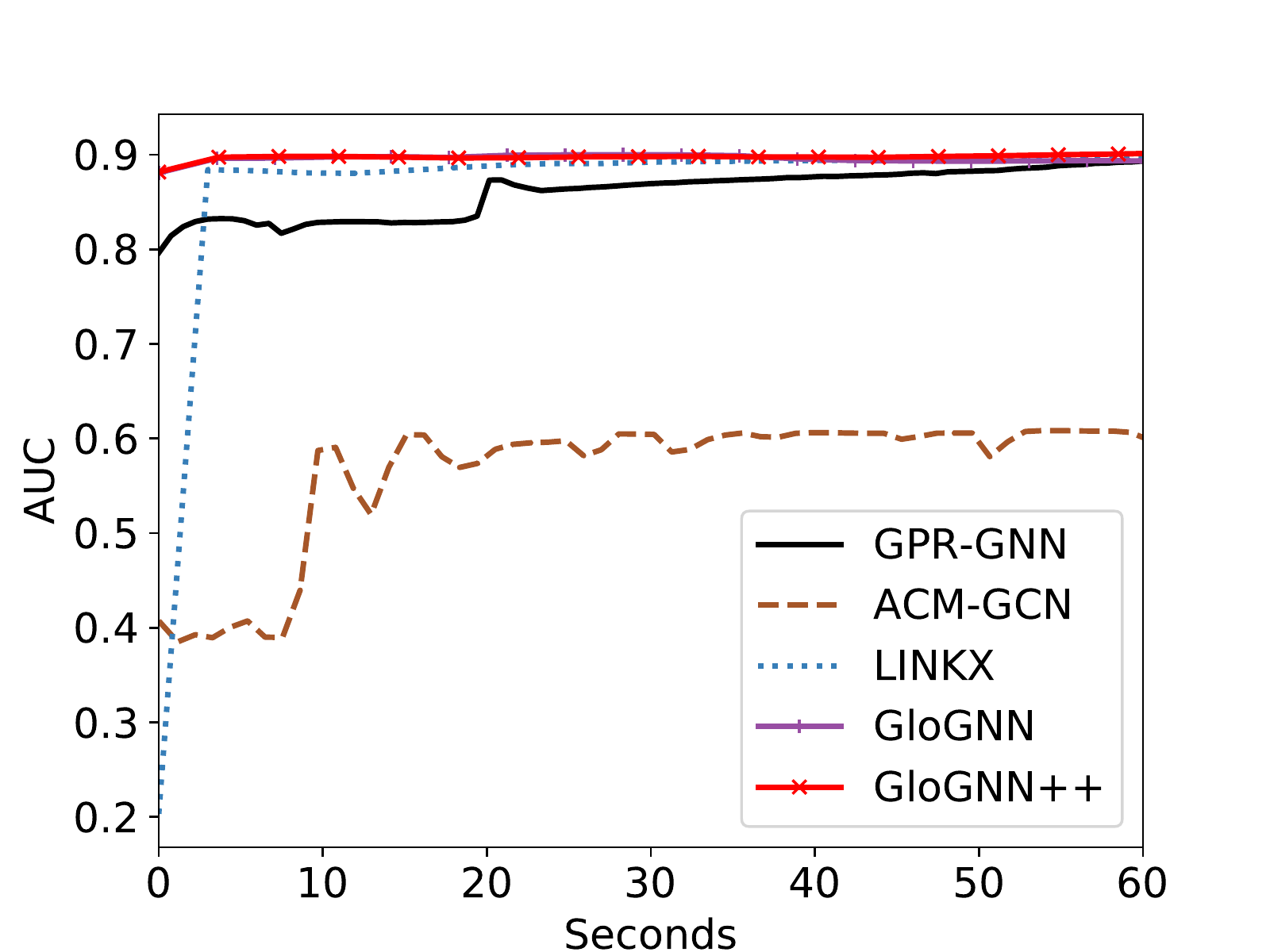}} 
    \subfigure[pokec]{\includegraphics[width=0.2\textwidth]{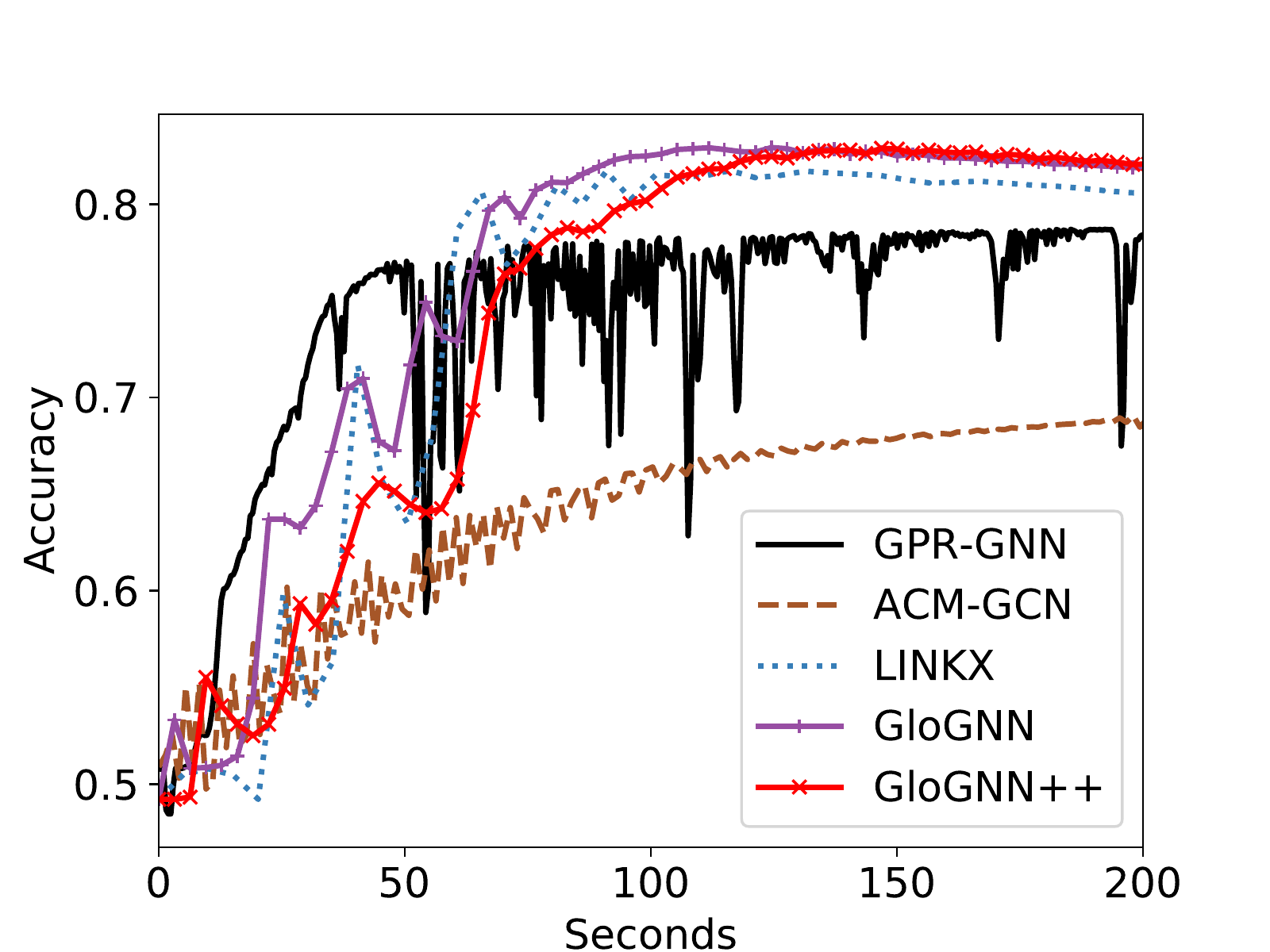}}    
    \subfigure[snap-patents]{\includegraphics[width=0.2\textwidth]{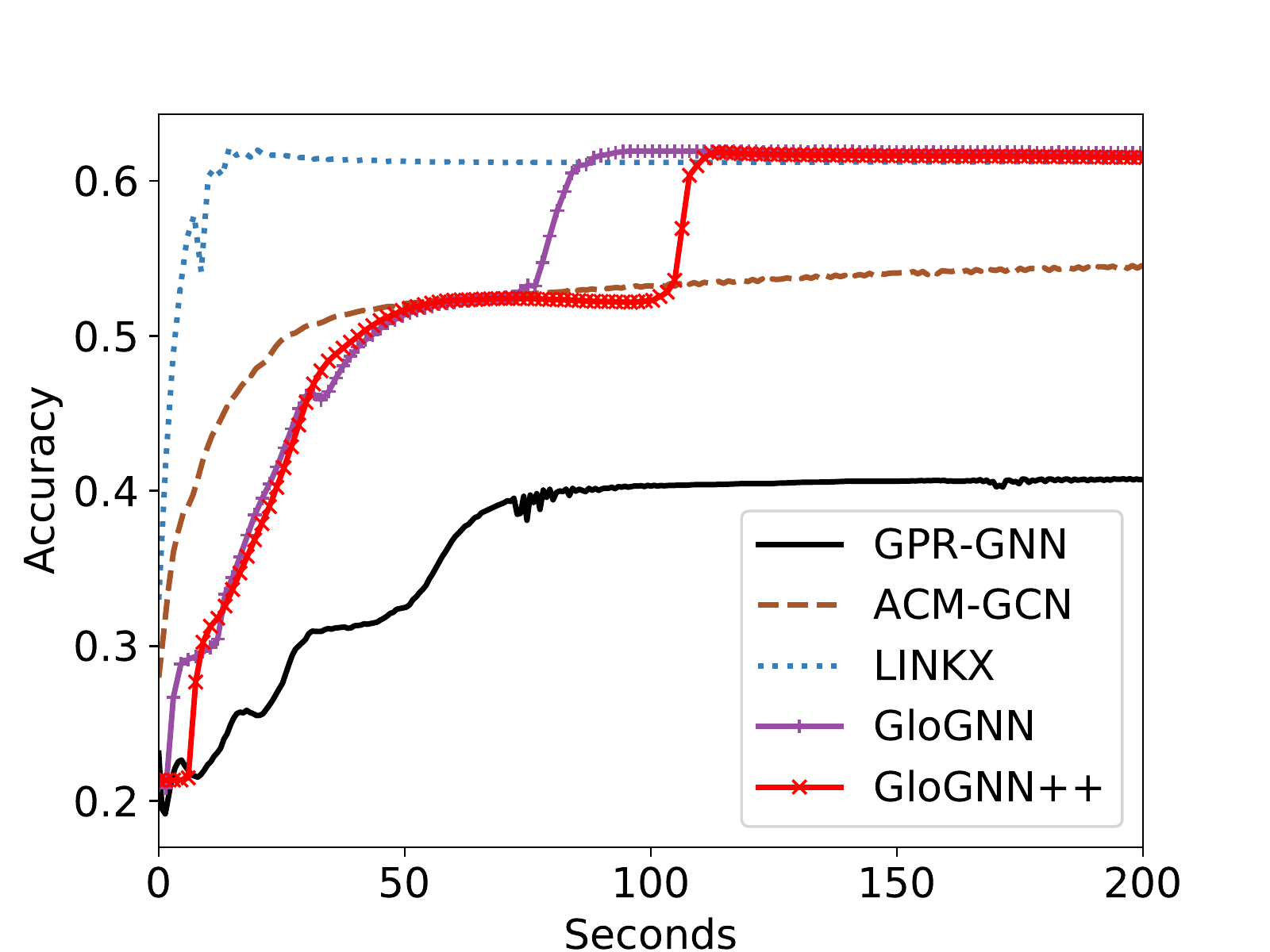}} 
    \subfigure[twitch-gamers]{\includegraphics[width=0.2\textwidth]{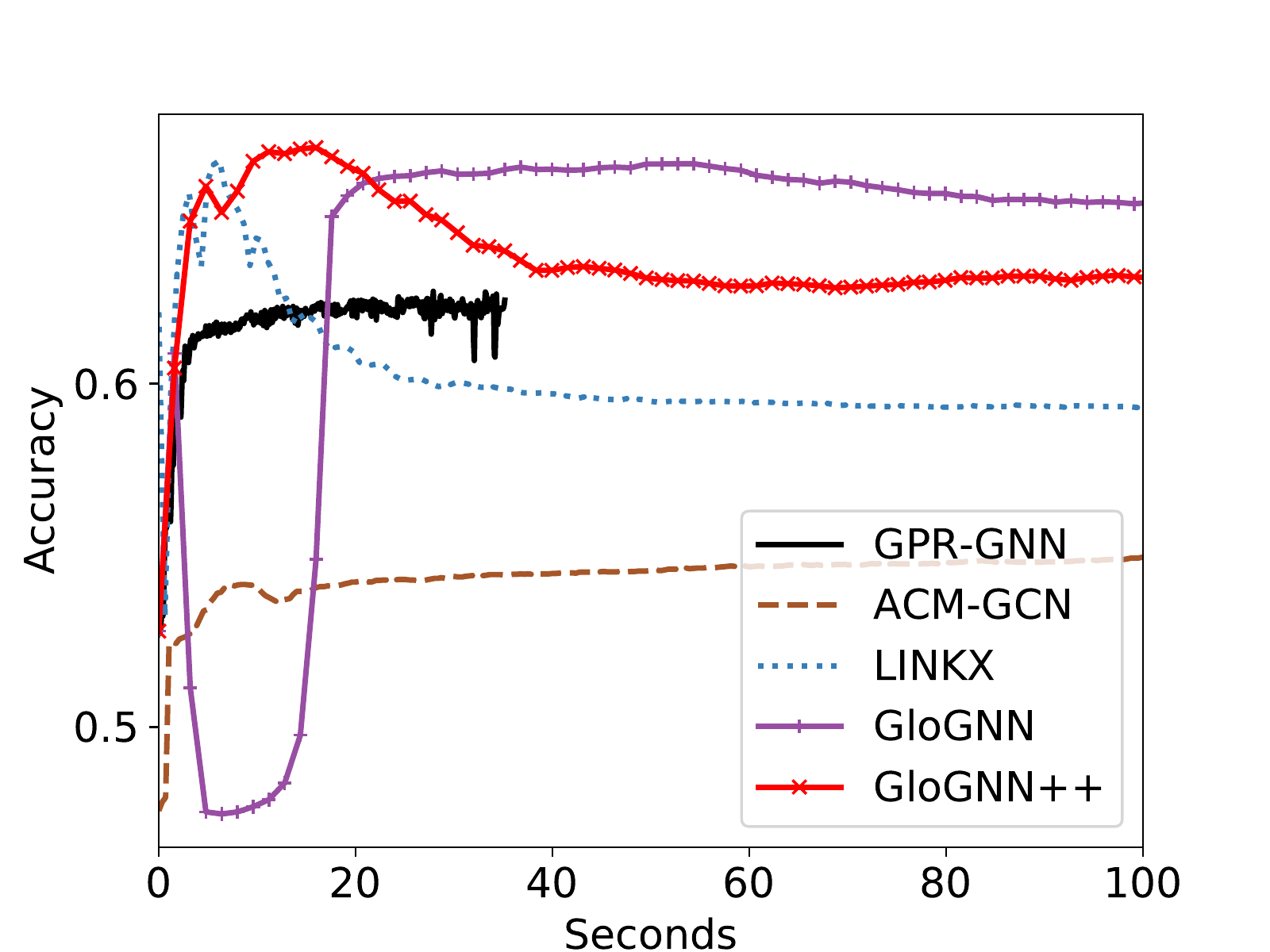}} 
     \caption{Efficiency study: x-axis shows the training time and y-axis is the accuracy/AUC score on the validation set.}
     \label{figure:runtime}
\end{figure}

From Fig.~\ref{figure:runtime}, 
we see that \ada\ and \ada++ converge very fast to the best/runner-up results over all the datasets.
While GPR-GNN runs faster,
it generally performs poorly. 
For LINKX,
the
MLP-based model structure instead of GNN-based explains its scalability.
However,
the 8th-ranked accuracy score on datasets in Table~\ref{tab:result_small}
restricts its wide usage.
For H$_2$GCN and ACM-GCN,
they are slower than \ada\ and \ada++.
For example,
\ada++ achieves almost $8\times$ speedup than ACM-GCN on genius;
it is also $2\times$ faster than H$_2$GCN on Penn94.
These results show that \ada\ and \ada++
are highly effective and also efficient;
hence they can be widely applied to large-scale datasets.

%ACM-GCN uses both low-pass and high-pass filters for each node in each layer
%while GPR-GNN only utilizes one convolutional filter.

\subsection{Grouping effect}

Lemma~\ref{lemma:z-star} shows that both the coefficient matrix $Z^*$ and the node embedding matrix $H$ have the desired grouping effect.
%we show the results on three datasets: Texas, Wisconsin and Cornell.
Considering the dataset size for clear illustration,
we choose Texas, Wisconsin and Cornell as representatives to show the grouping effect of $Z^*$ in Figure~\ref{figure:z} (a)-(c).
%Figure~\ref{figure:z} visually illustrates the grouping effect of $Z^*$ on Texas, Wisconsin and Cornell.
All these datasets contain nodes in five labels.
%which consist of nodes of five labels.
In 
each sub-figure, 
rows and columns are reordered by ground-truth labels.
We use 
red and blue to indicate positive and negative values, respectively.
We further use pixel color brightness to show the positive/negative degree of a value.
The brighter a pixel,
the larger the degree. 
From the figures,
the matrix exhibits the well-defined block diagonal structure.
This shows the grouping effect of $Z^*$.
For Texas and Cornell,
%they include nodes with five different labels but
we see only four blocks along the diagonal.
This 
is because in both datasets, 
there exists one object class that includes only one node.
Similarly,
Figure~\ref{figure:z} (d)-(f) further show the grouping effect of 
the output node embedding matrix $H$ on these datasets.
We reorder columns by gold-standard classes.
Each column in the matrix corresponds to a node's embedding vector.
For nodes in the same class,
their embedding vectors are close to each other.
This further explains the superior performance of our models.

\begin{figure}[!htbp]
    \centering
     \subfigure[Texas]{\includegraphics[height=0.15\textwidth]{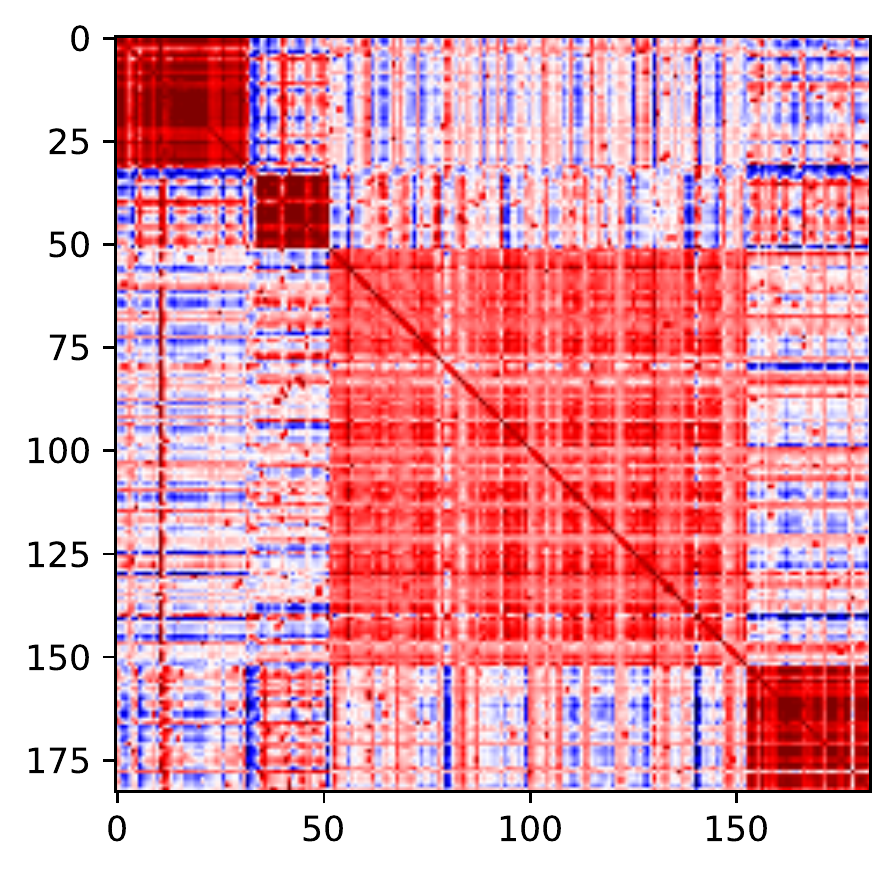}}    
    \subfigure[Wisconsin]{\includegraphics[height=0.15\textwidth]{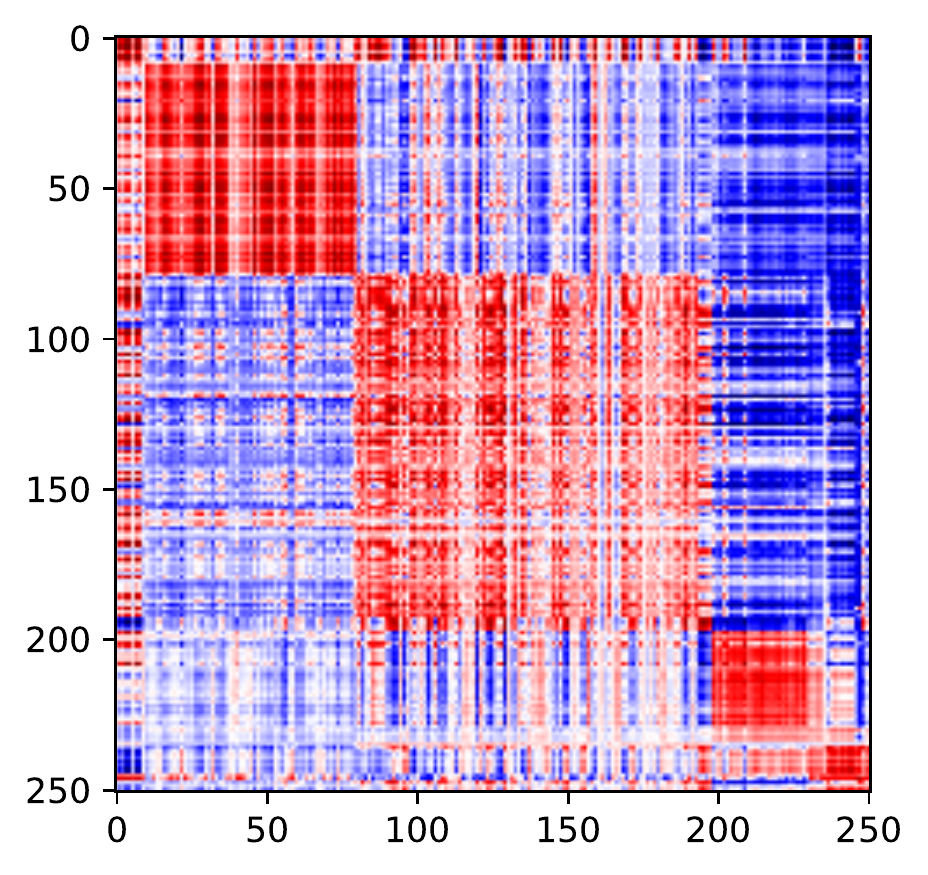}} 
    \subfigure[Cornell]{\includegraphics[height=0.15\textwidth]{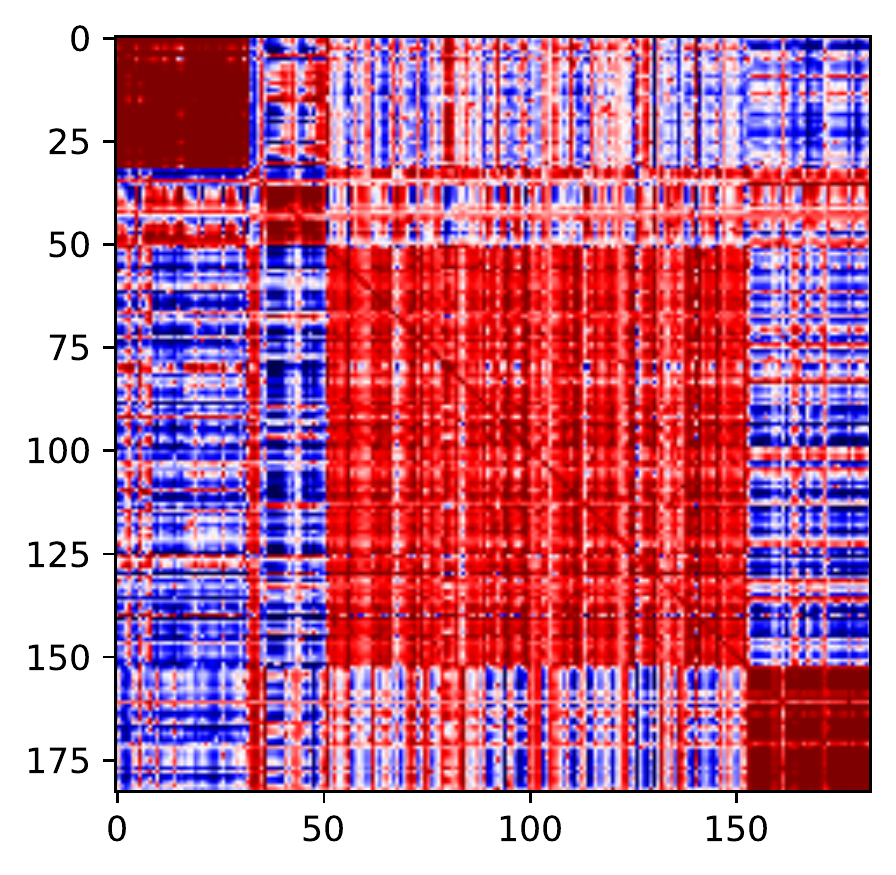}} 
     \subfigure[Texas]{\includegraphics[width=0.5\textwidth]{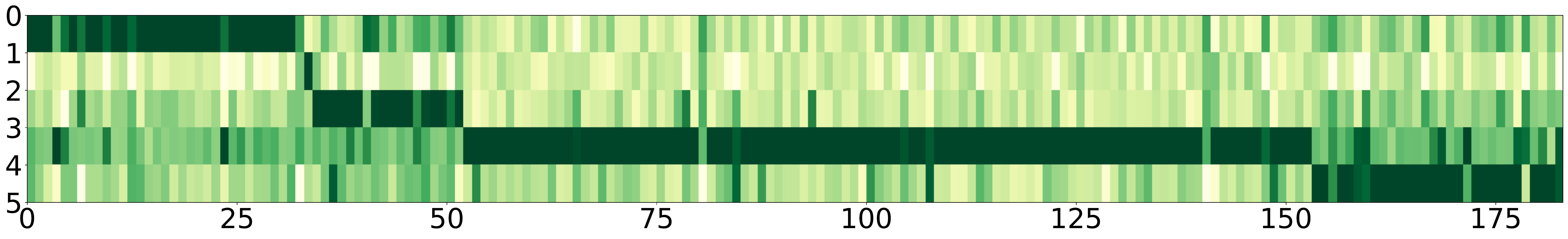}}    
    \subfigure[Wisconsin]{\includegraphics[width=0.5\textwidth]{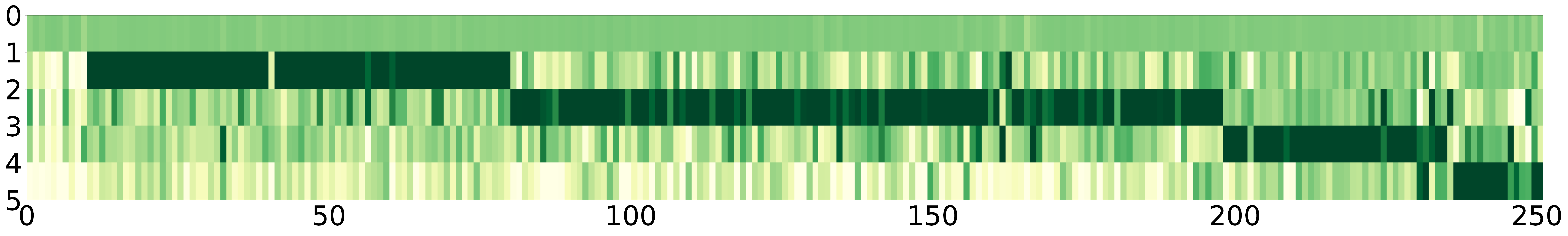}} 
    \subfigure[Cornell]{\includegraphics[width=0.5\textwidth]{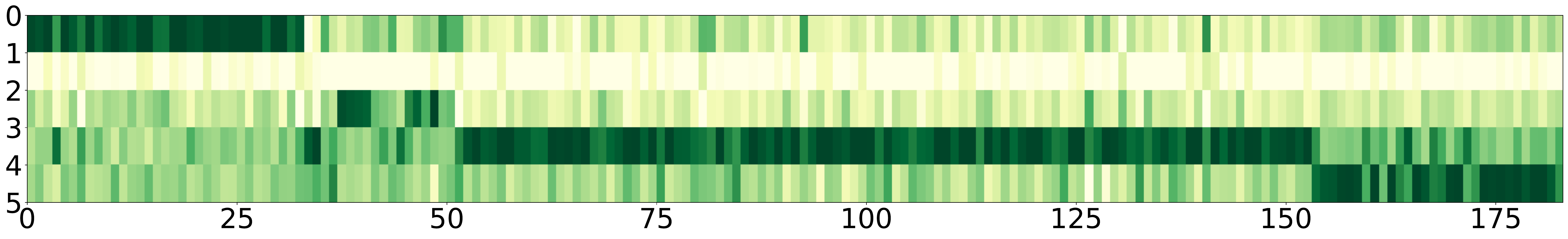}} 
     \caption{The grouping effect of $Z^*$ (a)-(c) and $H$ (d)-(f) on Texas, Wisconsin and Cornell (better view in color).}
     \label{figure:z}
\end{figure}

%\begin{figure}[!htbp]
%    \centering
%     \subfigure[Texas]{\includegraphics[width=0.5\textwidth]{fig/texas_h.pdf}}    
%    \subfigure[Wisconsin]{\includegraphics[width=0.5\textwidth]{fig/wisconsin_h.pdf}} 
%    \subfigure[Cornell]{\includegraphics[width=0.5\textwidth]{fig/cornell_h.pdf}} 
%     \caption{The grouping effect of $H$ on Texas, Wisconsin and Cornell. Each column corresponds to a node's embedding vector.}
%     \label{figure:h}
%\end{figure}

\subsection{Global homophily}
We end this section with a study to show how \ada\ finds global homophily for nodes in the graph.
Given a graph,
we first calculate the average number of $k$-hop neighbors that share the same label with a node. 
We further inspect the average number of positive $Z^*$ values for these neighbors.
After that,
we compare the results on 6 graphs with large heterophily in Figure~\ref{figure:globalfriends}.
We see that
for each node in these datasets,
the average number of adjacent neighbors in the same class
is less than that of multi-hop
ones (2-hop to 6-hop).
There also exist many $>6$-hop neighbors that can be used to predict a node's label.
This necessitates jumping the locality of a node and finding its global homophily.
Further,
for each node,
our model \ada\ can correctly assign positive values to the global nodes in the same class,
including both adjacent neighbors
and those that are distant.
This also explains the effectiveness of our models.

\begin{figure}[!htbp]
    \centering  
    \subfigure[Texas]{\includegraphics[width=0.15\textwidth]{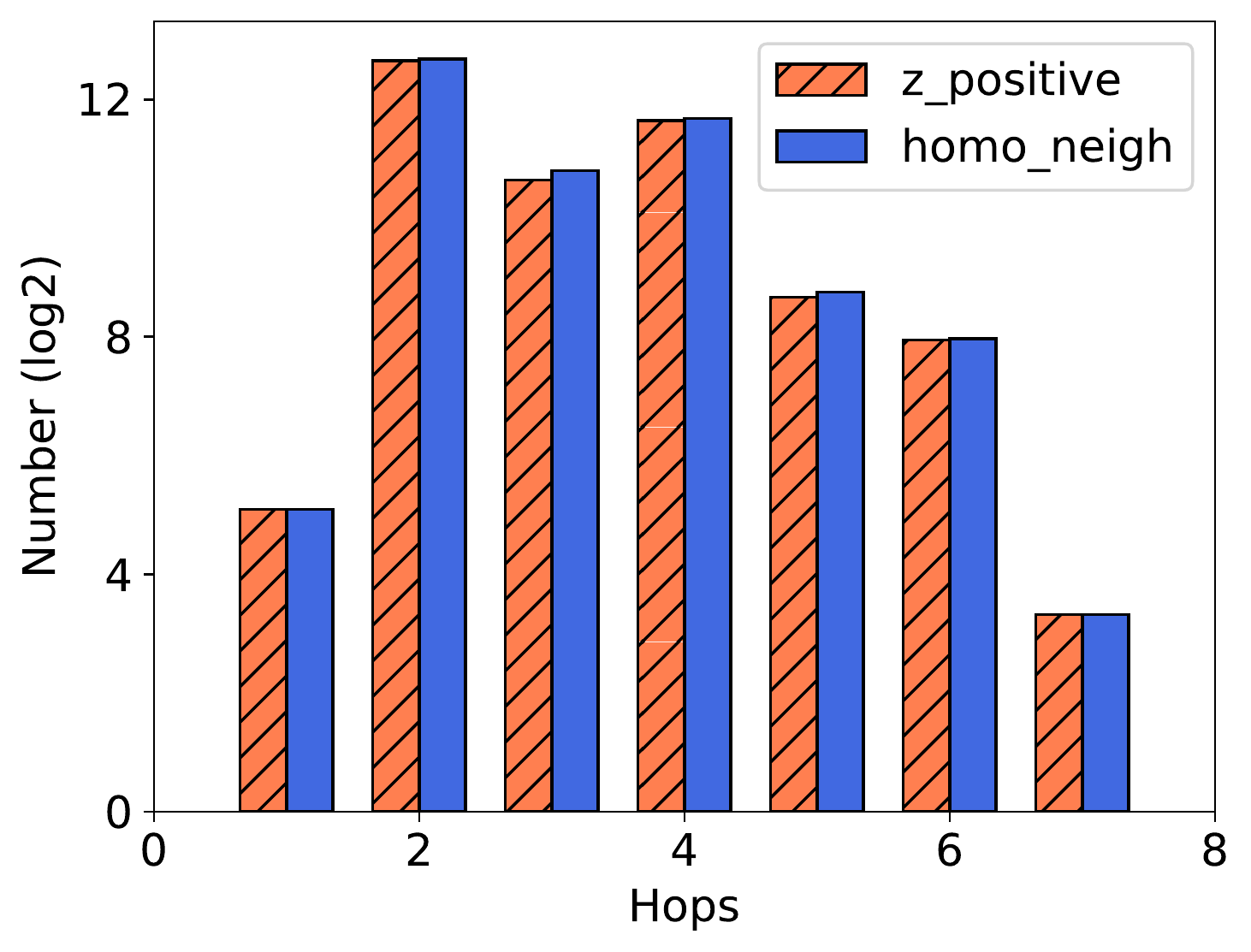}} 
    \subfigure[Wisconsin]{\includegraphics[width=0.15\textwidth]{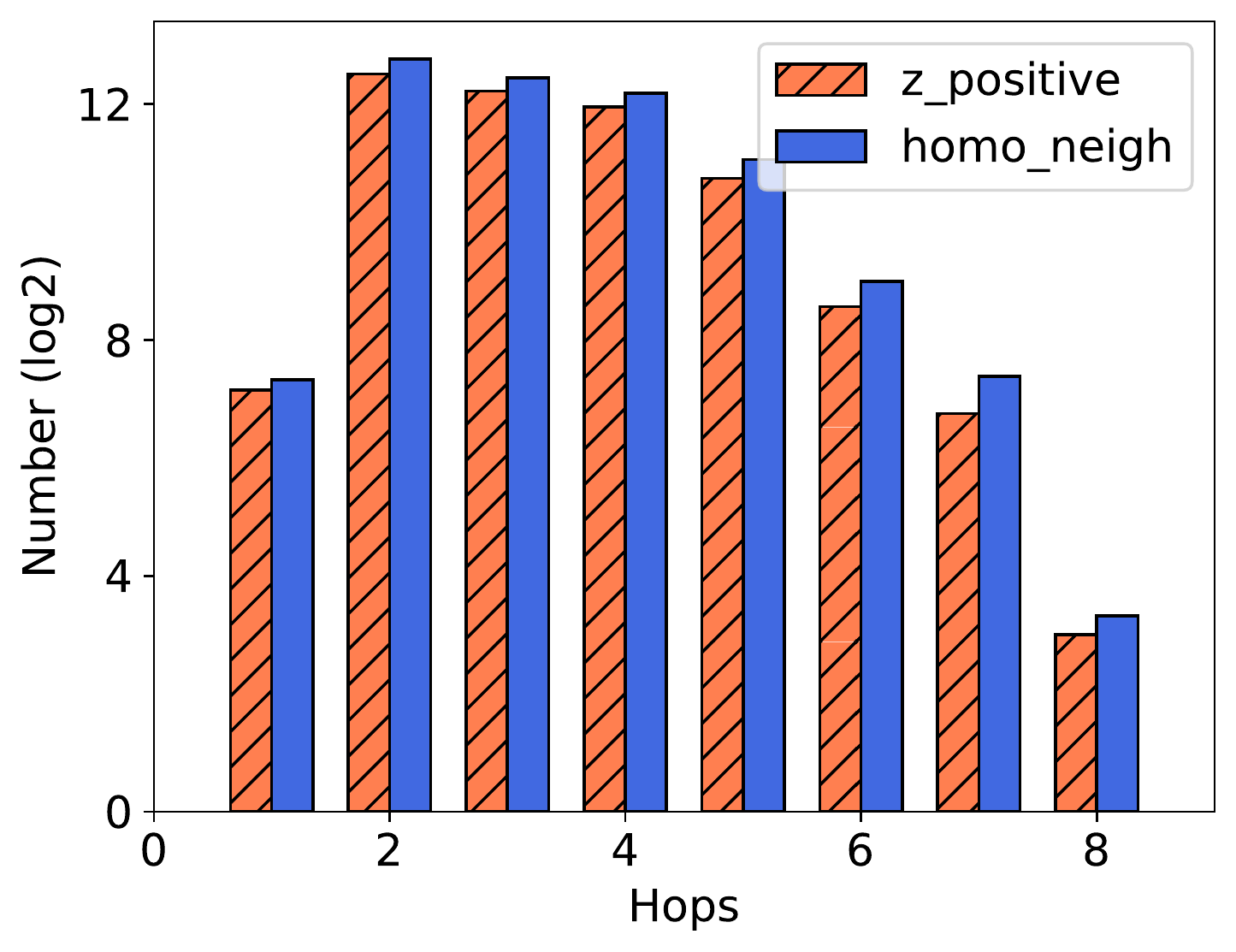}} 
     \subfigure[Cornell]{\includegraphics[width=0.15\textwidth]{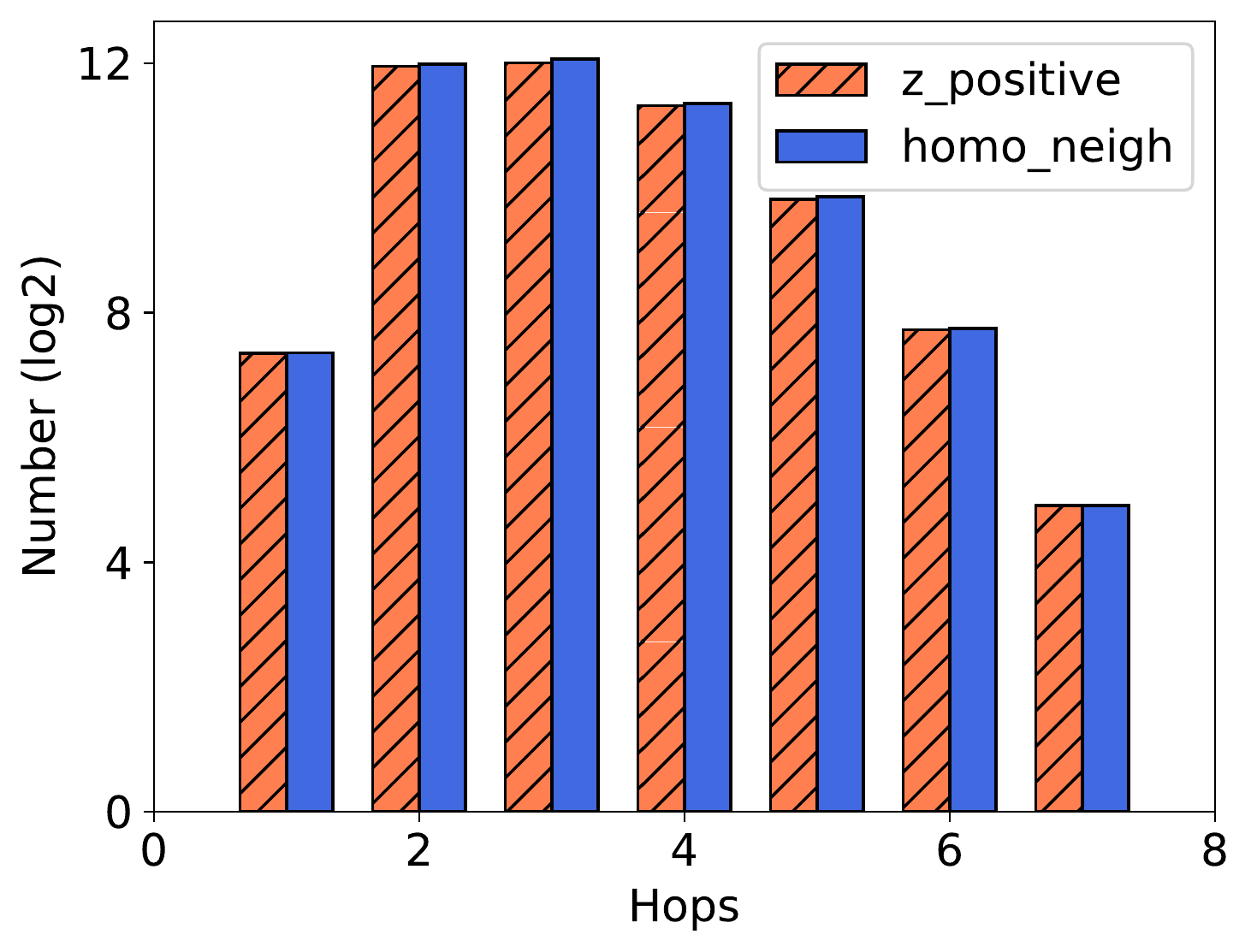}}
     \subfigure[Actor]{\includegraphics[width=0.15\textwidth]{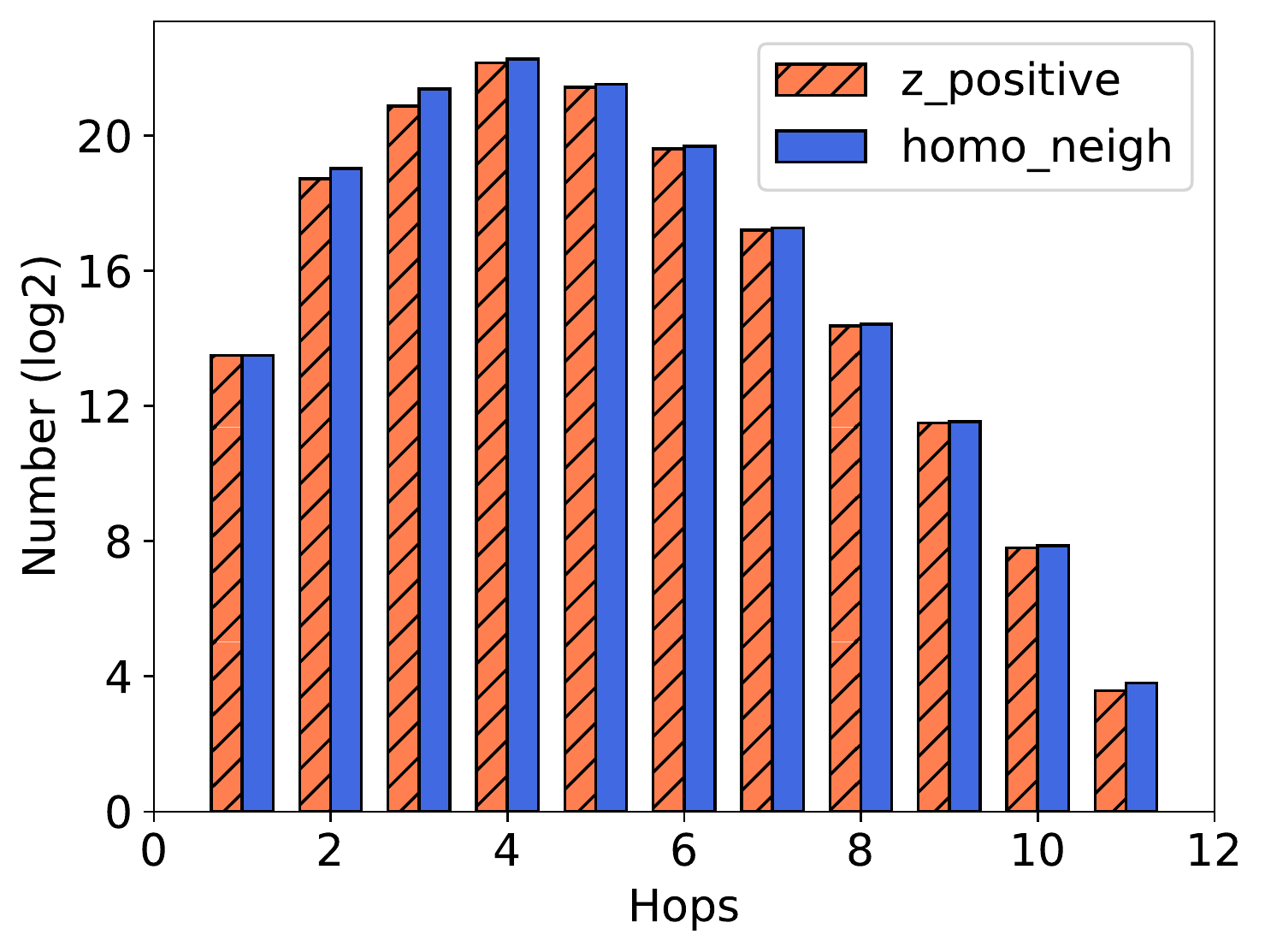}}  
     \subfigure[Squirrel]{\includegraphics[width=0.15\textwidth]{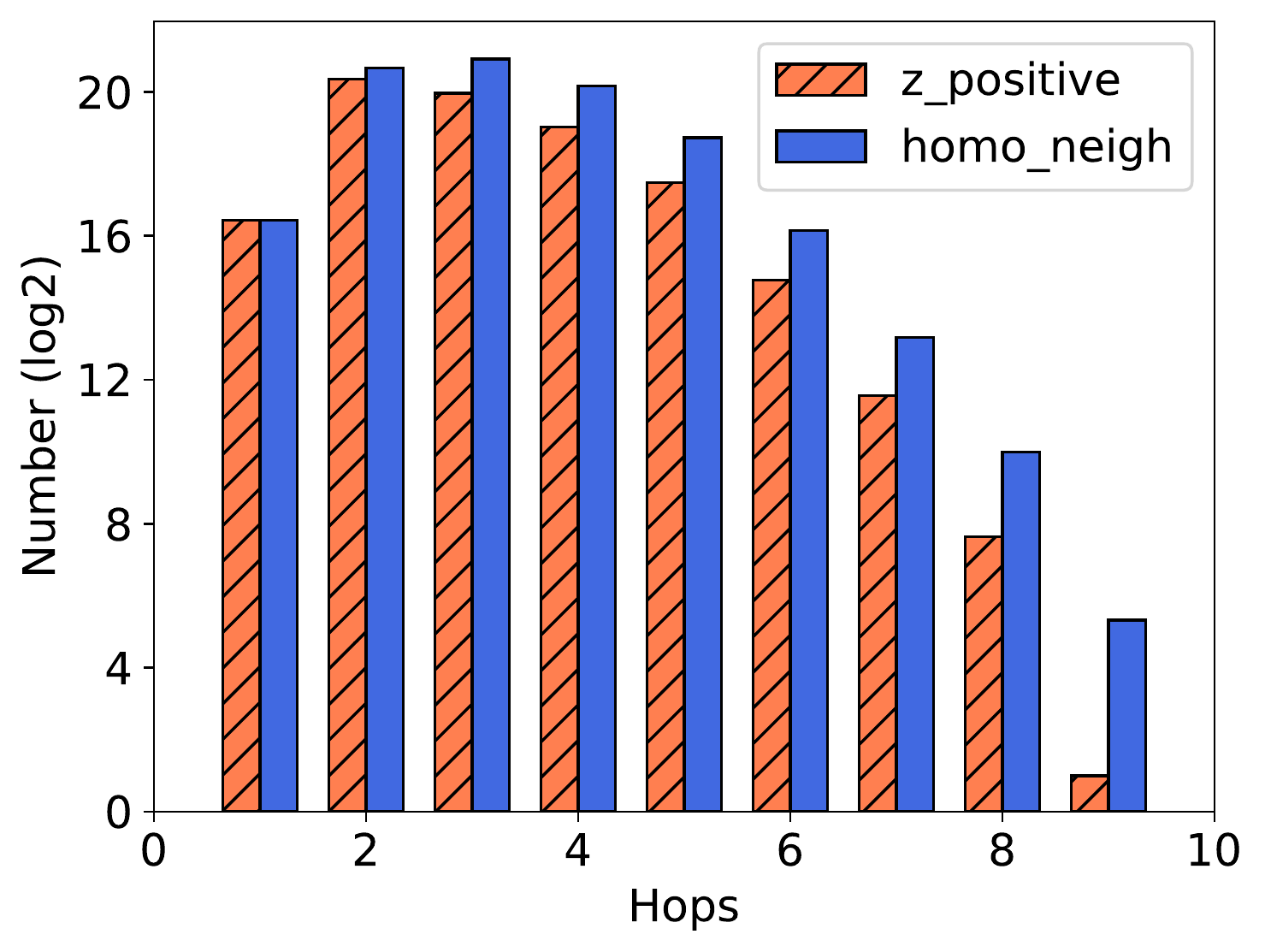}} 
      \subfigure[Chameleon]{\includegraphics[width=0.15\textwidth]{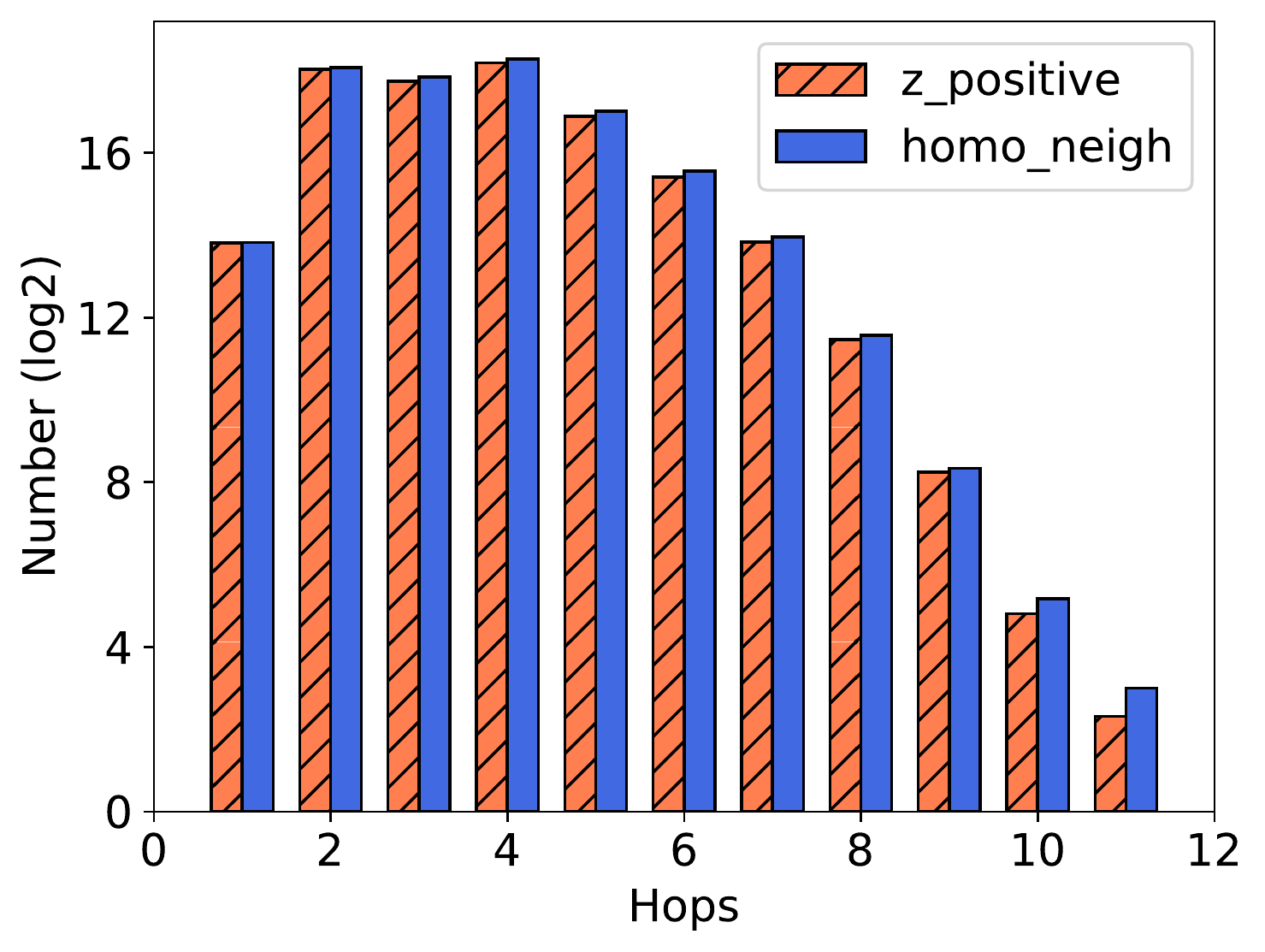}}     
     \caption{Global homophily study}
     \label{figure:globalfriends}
\end{figure}

\section{Conclusions}
\label{sec:conclusion}
In this paper,
we generalized GNNs to graphs with heterophily. 
We 
proposed \ada\ and \ada++,
which 
generate a node's embedding by aggregating information from global nodes in the graph.
In each layer,
we formulated an optimization problem to derive
a coefficient matrix $Z$ that describes the relationships between nodes.
%Since $Z$ allows signed values,
%the aggregation operation based on $Z$ automatically combine low-pass and high-pass convolutional filters.
Neighborhood aggregation is then performed based on $Z$.
We accelerated the aggregation process by matrix multiplication reordering without explicitly calculating $Z$.
We mathematically proved that both $Z$ and the generated node embedding matrix $H$ have the desired grouping effect,
which explains the model effectiveness.
%We further upgraded the model and put forward \ada++,
%which employs an attention mechanism to learn the importance of hidden embedding features of nodes.
We conducted extensive experiments to evaluate the performance of our models.
Experimental results show that our methods performs favorably against other 11 competitors over 15 datasets of diverse heterophilies;
they are also efficient
and converge very fast.

\clearpage 

\bibliography{sample-base}
\bibliographystyle{icml2022}

%%%%%%%%%%%%%%%%%%%%%%%%%%%%%%%%%%%%%%%%%%%%%%%%%%%%%%%%%%%%%%%%%%%%%%%%%%%%%%%
%%%%%%%%%%%%%%%%%%%%%%%%%%%%%%%%%%%%%%%%%%%%%%%%%%%%%%%%%%%%%%%%%%%%%%%%%%%%%%%
% APPENDIX
%%%%%%%%%%%%%%%%%%%%%%%%%%%%%%%%%%%%%%%%%%%%%%%%%%%%%%%%%%%%%%%%%%%%%%%%%%%%%%%
%%%%%%%%%%%%%%%%%%%%%%%%%%%%%%%%%%%%%%%%%%%%%%%%%%%%%%%%%%%%%%%%%%%%%%%%%%%%%%%

\appendix
\onecolumn
%\section{You \emph{can} have an appendix here.}
%
%You can have as much text here as you want. The main body must be at most $8$ pages long.
%For the final version, one more page can be added.
%If you want, you can use an appendix like this one, even using the one-column format.
%%%%%%%%%%%%%%%%%%%%%%%%%%%%%%%%%%%%%%%%%%%%%%%%%%%%%%%%%%%%%%%%%%%%%%%%%%%%%%%
%%%%%%%%%%%%%%%%%%%%%%%%%%%%%%%%%%%%%%%%%%%%%%%%%%%%%%%%%%%%%%%%%%%%%%%%%%%%%%%

\section{Pseudocodes}
\label{sec:sup}

%\noindent \textbf{[Semi-supervised node classification]}.
Given a graph $\mathcal{G} =(\mathcal{V}, \mathcal{E})$ and a label set $\mathcal{C}$ with $|\mathcal{C}| = c$,
let $\mathcal{V} = \mathcal{L} \cup \mathcal{U}$,
where $\mathcal{L}$ is a set of labeled objects and $\mathcal{U}$ is a set of unlabeled
ones, the node classification problem is to learn a mapping $\psi$: $\mathcal{V}  \rightarrow \mathcal{C}$ to 
predict the labels
of nodes in $\mathcal{U}$.
We next summarize the pseudocodes of \ada\ as follows.

\begin{algorithm}
\caption{\ada}
\label{alg}
\begin{algorithmic}[1]
 \STATE {\bfseries Input:} $\mathcal{G} = (\mathcal{V}, \mathcal{E})$, $\mathcal{V} = \mathcal{L} \cup \mathcal{U}$, $A$, $X$, $L$, $\mathcal{C}$, $Y_\mathcal{L}$
 \STATE {\bfseries Output:} the label matrix $Y_\mathcal{U}$ of unlabeled nodes
%\State Run \emph{metapath2vec} with $\mathcal{P}$
\STATE Calculate $H_X^{(0)}$ and $H_A^{(0)}$ by Eq.~\ref{eq:h0xa}
\STATE Calculate $H^{(0)}$ by Eq.~\ref{eq:h0}
\FOR {$l \leftarrow 0$ to $L-1$}
\STATE Calculate $Q^{(l+1)}$ by Eq.~\ref{eq:q}
\STATE Calculate $H^{(l+1)}$ by Eq.~\ref{eq:Hexpansion}
\ENDFOR
\STATE Normalize $H^{(L)}$ with the \texttt{Softmax} function and feed the results into the \texttt{Cross-entropy} function
\STATE Optimize the objective function to update weight matrices
\STATE {\bfseries Return:} $Y_\mathcal{U}$
\end{algorithmic}
\end{algorithm}

\section{Datasets}
\label{sec:datasets}
We first use 9 small-scale datasets from~\cite{pei2020geom} 
and divide them into the following four categories:

\noindent{\small$\bullet$}
\textbf{[Citation network]}.
\emph{Cora}, 
\emph{Citeseer} and \emph{Pubmed} are citation graphs,
where each node represents a scientific paper.
These graphs use 
bag-of-words representations as the feature vectors of nodes.
Each node is assigned a label indicating the research field.
Note that these three datasets are homophilous graphs.

\noindent{\small$\bullet$}
\textbf{[WebKB]}.
%WebKB is a webpage dataset collected from computer science departments of various universities by CMU. 
\emph{Texas}, \emph{Wisconsin} and \emph{Cornell} 
are web page datasets collected 
from computer science departments of various universities.
In these datasets, 
nodes are web pages and 
edges represent hyperlinks between them.
We take bag-of-words representations as nodes' feature vectors.
The task is to classify the web pages into five categories including \emph{student}, \emph{project}, \emph{course}, \emph{staff} and \emph{faculty}.

\noindent{\small$\bullet$}
\textbf{[Actor co-occurrence network]}.
\emph{Actor} is a graph induced from the film-director-actor-writer network in~\cite{tang2009social},
which describes the co-occurrence relation between actors in Wikipedia pages.
Node features are constructed by keywords contained in the Wikipedia pages of actors.
The task is to classify actors into five categories.

\noindent{\small$\bullet$}
\textbf{[Wikipedia network]}.
\emph{Squirrel} and \emph{Chameleon}
are two subgraphs of web pages in Wikipedia.
Our task is to classify nodes into five categories based on their average amounts of monthly traffic.

To further show the effectiveness and efficiency of our models,
we also use 6 large-scale datasets released by~\cite{lim2021large}:

\noindent{\small$\bullet$}
\textbf{[Social network]}.
%Penn94, Pokec, genius and twitch-gamers are four online social networks.
%Specifically,
\emph{Penn94} is a subgraph extracted from Facebook whose nodes are students.
Node features include major, second major/minor, dorm/house, year and high school.
We take students' genders as nodes' labels. 
\emph{Pokec} is a friendship network from a Slovak online social network,
whose nodes are users and edges represent directed friendship relations.
We construct node features from users' profiles, 
such as geographical region, registration time, age.
The task is to classify users based on their genders.
\emph{genius} is a subnetwork extracted from genius.com,
which is a website for crowdsourced annotations of song lyrics.
In the graph,
nodes are users and edges connect users that follow each other.
User features include expertise scores, 
counts of contributions, roles held by users, etc.
Some users are marked with a ``gone'' label on the site, 
which are more likely to be spam users.
Our goal is to predict whether a user is marked with ``gone''.
%The binary classification task is to predict whether a user is mark ``gone''.
\emph{twitch-gamers}
is a subgraph from the streaming platform Twitch,
where nodes are users and edges connect mutual followers.
Node features include 
the number of views,
the creation and update dates,
language,
life time 
and whether the account is dead.
The task is to predict whether the channel has explicit content.

\noindent{\small$\bullet$}
\textbf{[Citation network]}.
\emph{arXiv-year} is a directed subgraph of ogbn-arXiv, 
where nodes are arXiv papers and edges represent the citation relations.
We construct node features by 
taking the averaged word2vec embedding vectors of tokens 
contained in both the title and abstract of papers.
The task is to classify these papers into five labels that are constructed based on their posting year.
\emph{snap-patents}
is a US patent network whose nodes are patents and edges are citation relations.
Node features are constructed from patent metadata.
Our goal is to classify the patents into five labels based on the time when they were granted.

\section{Aggregation acceleration}
\label{sec:agg}
To accelerate the updates of $H^{(l+1)}$ in Equation~\ref{eq:conv},
we first follow the Woodbury formula~\cite{max1950inverting} to derive
\begin{small}
\begin{equation}
\begin{split}
\label{eq:inverse}
%\scriptsize
 & \left [ (1-\gamma)^2 H^{(l)} (H^{(l)})^T +  (\beta_1+\beta_2) I_n\right ]^{-1}  \\
 & =  \frac{1}{\beta_1+\beta_2} I_n - \frac{1}{(\beta_1+\beta_2)^2} H^{(l)} \left[ \frac{1}{(1-\gamma)^2}I_c + \frac{1}{\beta_1+\beta_2} (H^{(l)})^TH^{(l)} \right]^{-1} (H^{(l)})^T \\
\end{split}
\end{equation}
\end{small}
After that,
based on Eq.~\ref{eq:zstar} and Eq.~\ref{eq:inverse},
we can easily transform Eq.~\ref{eq:conv} into Eq.~\ref{eq:Hexpansion}.

\section{Proof}
\label{sec:proof}

In this section, 
we prove Lemma~\ref{lemma2}, Lemma~\ref{lemma3} and Lemma~\ref{lemma:z-star}, respectively.
In the following discussion,
we use 
$z_i^{(l)*}$ to denote the $i$-th row of $Z^{(l)*}$,
which is the coefficient vector for representing node $v_i$;
%For a node $v_i$,
we denote $\hat{a}_i^k$ as
the $i$-th row of $\hat{A}^k$,
which represents $v_i$'s $k$-hop node reachability in a graph.
We first consider Lemma~\ref{lemma1}:
\begin{lemma}
\label{lemma1}
%Given a set of nodes $\mathcal{V} = \{v_i\}_{i=1}^n$,
%the $l$-th layer embedding matrix $H^{(l)}$
%and a set of $k$-hop reachability matrices $\{\hat{A}^k\}_{k=0}^K$,
$\forall 1 \leq i, p \leq n$,
%the optimal solution 
the optimal solution $Z^{(l)*}$ in Eq.~\ref{eq:obj} satisfies
\begin{equation}
\label{eq:zi}
Z_{ip}^{(l)*} = \frac{(1-\gamma)[{h}_i^{(l)}-(1-\gamma)z_i^{(l)*} H^{(l)}-\gamma h_i^{(0)}] (h_p^{(l)})^T + \beta_2 \sum_{k=1}^{K}\lambda_{k}\hat{A}_{{ip}}^k}{\beta_1 + \beta_2}.
\end{equation}
\end{lemma}
\begin{proof}
For $1 \leq i \leq n$,
we define $J(z_i^{(l)}) =  \Vert {h}_i^{(l)}-(1-\gamma) z_i^{(l)} H^{(l)} -\gamma h_i^{(0)}\Vert_2^2 +\beta_1 \Vert z_i^{(l)} \Vert_2^2 +\beta_2 \Vert z_i^{(l)}-\sum_{k=1}^{K}\lambda_{k}\hat{a}_{i}^k\Vert_2^2$.
Since $Z^{(l)*}$ is the optimal solution of Equation~\ref{eq:obj}, 
we have $\frac{\partial{J}}{\partial{Z}_{ip}^{(l)}}\vert_{{z}_i^{(l)} = {z}_i^{(l)*}} = 0,\ \forall 1\leq p \leq n$.
We take the derivative and get 
$-(1-\gamma)[{h}_i^{(l)}-(1-\gamma) z_i^{(l)*} H^{(l)} -\gamma h_i^{(0)}] (h_p^{(l)})^T + \beta_1 Z_{ip}^{(l)*} + \beta_2 (Z_{ip}^{(l)*}-\sum_{k=1}^{K}\lambda_{k}\hat{A}_{{ip}}^k) = 0$,
which induces Eq.~\ref{eq:zi}.
\end{proof}

Based on Lemma~\ref{lemma1}, 
we first prove Lemma~\ref{lemma2}:
\begin{proof}
%Since the column vectors of $X$ are normalized (i.e., $\bm{x}_q^T \bm{x}_q = 1 \; \forall 1 \leq q \leq n$) , we have
%$||\bm{x}_i - \bm{x}_j||_2 = \sqrt{2(1-r)}$,
%where $r = \bm{x}_i^T\bm{x}_j$.
%measuring the closeness between $\bm{x}_i$ and $\bm{x}_j$ in the feature space.
%As $\bm{z}_p^*$ is the optimal solution, 
From Equation~\ref{eq:zi}, we get 
\begin{equation}
\begin{split}
Z_{ip}^{(l)*}-Z_{jp}^{(l)*} & = \frac{1-\gamma}{\beta_1+\beta_2}({h}_i^{(l)} - {h}_j^{(l)})(h_p^{(l)})^T - \frac{(1-\gamma)^2}{\beta_1+\beta_2} (z_i^{(l)*} - z_j^{(l)*}) H^{(l)}(h_p^{(l)})^T  \\
 &-\frac{\gamma(1-\gamma)}{\beta_1+\beta_2} (h_i^{(0)} - h_j^{(0)})(h_p^{(l)})^T + \frac{\beta_2}{\beta_1+\beta_2} \sum_{k=1}^{K}\lambda_{k}(\hat{A}_{{ip}}^k - \hat{A}_{{jp}}^k) \\
% & = \frac{1-\gamma}{\beta}({h}_i^{(l)} - {h}_j^{(l)})(h_p^{(l)})^T - \frac{(1-\gamma)^2}{\beta}(h_i^{(l)}-h_j^{(l)})(H^{(l)})^T R   \\
% & +(1-\gamma) \sum_{k=1}^{K}\lambda_{k}(\hat{a}_i^k - \hat{a}_j^k) R - \frac{\gamma(1-\gamma)^2}{\beta}(h_i^{(0)}- h_j^{(0)})(H^{(l)})^T R \\
% & -\frac{\gamma}{\beta} (h_i^{(0)} - h_j^{(0)})(h_p^{(l)})^T + \sum_{k=0}^{K}\lambda_{k}(\hat{A}_{{ip}}^k - \hat{A}_{{jp}}^k) \\
 \end{split}
\end{equation}
Since 
\begin{equation}
\begin{split}
z_i^{(l)*}   = 
& \left [(1-\gamma)h_i^{(l)}(H^{(l)})^T +\beta_2 \sum_{k=1}^{K}\lambda_{k}\hat{a}_i^k - \gamma(1-\gamma) h_i^{(0)}(H^{(l)})^T\right ]  \cdot\\
& \left [(1-\gamma)^2H^{(l)}(H^{(l)})^T + (\beta_1 + \beta_2) I_n\right ]^{-1}  \\
%||H^{(l)}-(1-\gamma)H^{(l)}Z-\gamma H^{(0)}||_F^2 + \alpha_1 ||Z||_F^2 + \alpha_2 ||Z-\sum_{k=0}^{K}\lambda_{k}\hat{A}^k||_F^2,
\end{split}
\end{equation}
and
\begin{equation}
\begin{split}
z_j^{(l)*}   = 
& \left [(1-\gamma)h_j^{(l)}(H^{(l)})^T +\beta_2 \sum_{k=1}^{K}\lambda_{k}\hat{a}_j^k - \gamma(1-\gamma) h_j^{(0)}(H^{(l)})^T\right ]  \cdot\\
& \left [(1-\gamma)^2H^{(l)}(H^{(l)})^T + (\beta_1 + \beta_2) I_n\right ]^{-1}  \\
%||H^{(l)}-(1-\gamma)H^{(l)}Z-\gamma H^{(0)}||_F^2 + \alpha_1 ||Z||_F^2 + \alpha_2 ||Z-\sum_{k=0}^{K}\lambda_{k}\hat{A}^k||_F^2,
\end{split}
\end{equation}
let $R = \left [(1-\gamma)^2H^{(l)}(H^{(l)})^T + (\beta_1 + \beta_2) I_n\right ]^{-1} H^{(l)} (h_p^{(l)})^T $ and
we derive
\begin{equation}
\begin{split}
Z_{ip}^{(l)*}-Z_{jp}^{(l)*} 
%& = \frac{1-\gamma}{\beta_1 + \beta_2}({h}_i^{(l)} - {h}_j^{(l)})(h_p^{(l)})^T - \frac{(1-\gamma)^2}{\beta_1 + \beta_2} (z_i^* - z_j^*) H^{(l)}(h_p^{(l)})^T  \\
%&-\frac{\gamma(1-\gamma)}{\beta_1 + \beta_2} (h_i^{(0)} - h_j^{(0)})(h_p^{(l)})^T + \frac{\beta_2}{\beta_1+\beta_2} \sum_{k=1}^{K}\lambda_{k}(\hat{A}_{{ip}}^k - \hat{A}_{{jp}}^k) \\
 & = \frac{1-\gamma}{\beta_1 + \beta_2}({h}_i^{(l)} - {h}_j^{(l)})(h_p^{(l)})^T - \frac{(1-\gamma)^3}{\beta_1 + \beta_2}(h_i^{(l)}-h_j^{(l)})(H^{(l)})^T R   \\
 & - \frac{\beta_2(1-\gamma)^2}{\beta_1+\beta_2} \sum_{k=1}^{K}\lambda_{k}(\hat{a}_i^k - \hat{a}_j^k) R + \frac{\gamma(1-\gamma)^3}{\beta_1+\beta_2}(h_i^{(0)}- h_j^{(0)})(H^{(l)})^T R \\
 & -\frac{\gamma(1-\gamma)}{\beta_1 + \beta_2} (h_i^{(0)} - h_j^{(0)})(h_p^{(l)})^T +   \frac{\beta_2}{\beta_1+\beta_2}  \sum_{k=1}^{K}\lambda_{k}(\hat{A}_{{ip}}^k - \hat{A}_{{jp}}^k) \\
 \end{split}
\end{equation}
We further have 
\begin{equation}
\label{eq:ip_app}
\begin{split}
|Z_{ip}^{(l)*}-Z_{jp}^{(l)*}| & \leq \frac{1-\gamma}{\beta_1+\beta_2}\Vert {h}_i^{(l)} - {h}_j^{(l)} \Vert_2 \Vert (h_p^{(l)})^T - (1-\gamma)^2(H^{(l)})^T R \Vert_2   \\
 & + \frac{\gamma(1-\gamma)}{\beta_1 + \beta_2} \Vert h_i^{(0)} - h_j^{(0)}\Vert_2 \Vert (h_p^{(l)})^T -  (1-\gamma)^2(H^{(l)})^T R   \Vert_2 \\
 & + \frac{\beta_2(1-\gamma)^2}{\beta_1 + \beta_2} \sum_{k=1}^{K}\lambda_{k}\Vert \hat{a}_i^k - \hat{a}_j^k\Vert_2 \Vert R \Vert_2 + \frac{\beta_2}{\beta_1+\beta_2}  \sum_{k=1}^{K}\lambda_{k}|\hat{A}_{{ip}}^k - \hat{A}_{{jp}}^k| \\
\end{split}
\end{equation}
\end{proof}

We next prove Lemma~\ref{lemma3}:

\begin{proof}
From Equation~\ref{eq:zi}, we get 
\begin{equation}
\begin{split}
Z_{pi}^{(l)*}-Z_{pj}^{(l)*} & = \frac{(1-\gamma)[{h}_p^{(l)} -(1-\gamma) z_p^{(l)*} H^{(l)} -\gamma h_p^{(0)} ](h_i^{(l)} - h_j^{(l)})^T}{\beta_1 + \beta_2} \\
 &+  \frac{\beta_2}{\beta_1+\beta_2}  \sum_{k=1}^{K}\lambda_{k}(\hat{A}_{{pi}}^k - \hat{A}_{{pj}}^k) \\
\end{split}
\end{equation}
That implies 
\begin{equation}
%\nonumber
\label{eq:zizj}
\begin{split}
|Z_{pi}^{(l)*}-Z_{pj}^{(l)*}|
%& \leq \frac{(1-\gamma)|(h_i^{(l)} - h_j^{(l)})^T({h}_p^{(l)}-(1-\gamma)H^{(l)} z_p^*-\gamma h_p^{(0)})|}{\beta}  \\
% &+ \sum_{k=0}^{K}\lambda_{k}|\hat{A}_{{ip}}^k - \hat{A}_{{jp}}^k| \\
& \leq \frac{(1-\gamma) \lVert {h}_p^{(l)} -(1-\gamma) z_p^{(l)*} H^{(l)} -\gamma h_p^{(0)} \rVert_2 \lVert h_i^{(l)} - h_j^{(l)}\rVert_2}{\beta_1 + \beta_2}  \\
 &+  \frac{\beta_2}{\beta_1+\beta_2} \sum_{k=1}^{K}\lambda_{k}|\hat{A}_{{pi}}^k - \hat{A}_{{pj}}^k| \\
\end{split}
\end{equation}
Since $Z^{(l)*}$ is the optimal solution to Equation~\ref{eq:obj}, we have
\begin{equation}
%\nonumber
\begin{split}
& J(z_p^{(l)*})  =  \lVert {h}_p^{(l)}-(1-\gamma)z_p^{(l)*} H^{(l)}-\gamma h_p^{(0)}\rVert_2^2 + \beta_1\Vert z_p^{(l)*} \Vert_2^2 \\
& + \beta_2 \lVert z_p^{(l)*}-\sum_{k=1}^{K}\lambda_{k}\hat{a}_{p}^k\rVert_2^2 \leq  
J(0)  = \lVert{h}_p^{(l)}-\gamma h_p^{(0)}\rVert_2^2 + \beta_2 \lVert\sum_{k=1}^{K}\lambda_{k}\hat{a}_{p}^k\rVert_2^2 .\\
\end{split}
\end{equation}
Hence, 
\begin{equation}
%\nonumber
\lVert {h}_p^{(l)}-(1-\gamma)z_p^{(l)*}H^{(l)}-\gamma h_p^{(0)}\rVert_2 \leq \sqrt{\lVert{h}_p^{(l)}-\gamma h_p^{(0)}\rVert_2^2 + \beta_2 \lVert\sum_{k=1}^{K}\lambda_{k}\hat{a}_{p}^k\rVert_2^2} = \eta.
\end{equation}
Equation~\ref{eq:zizj} can be further simplified as
\begin{equation}
\label{eq:pi_app}
|Z_{pi}^{(l)*}-Z_{pj}^{(l)*}| \leq  \frac{\eta (1-\gamma) \lVert h_i^{(l)} - h_j^{(l)}\rVert_2 
 +   \beta_2 \sum_{k=1}^{K}\lambda_{k}|\hat{A}_{{pi}}^k - \hat{A}_{{pj}}^k|}{\beta_1 + \beta_2} 
\end{equation}
\end{proof}

The proof of Lemma~\ref{lemma:z-star} is given as follows:
\begin{proof}
%We next deduce an inductive proof.
Given two nodes $v_i$ and $v_j$, if $v_i \rightarrow v_j$,
we can get by definition
(1) $\Vert x_i - x_j\Vert_2 \rightarrow 0$
and (2) $\lVert \hat{a}_i^k - \hat{a}_j^k\rVert_2 \rightarrow 0,\ \forall k \in [1,K]$.
Then based on Equations~\ref{eq:h0xa} and~\ref{eq:h0},
we can easily get $\Vert h_i^{(0)} - h_j^{(0)}\Vert_2 \rightarrow 0$.
Hence $H^{(0)} $ has grouping effect.
We next show that $Z^{(0)*}$ has grouping effect.
Since $\lVert \hat{a}_i^k - \hat{a}_j^k\rVert_2 \rightarrow 0$, 
then $|\hat{A}_{{ip}}^k - \hat{A}_{{jp}}^k| \rightarrow 0$ and
$|\hat{A}_{{pi}}^k - \hat{A}_{{pj}}^k| \rightarrow 0$ (due to the symmetry of $\hat{A}^k$).
According to Equation~\ref{eq:ip_app},
the R.H.S. of the equation will become close to 0, 
which induces that $|Z_{ip}^{(0)*}-Z_{jp}^{(0)*}|\rightarrow 0$
and
$Z^{(0)*}$ thus has grouping effect.
Similarly,
the R.H.S. of Equation~\ref{eq:pi_app} also approaches 0,
leading to the grouping effect of $(Z^{(0)*})^T$.
Then we show $H^{(1)}$ has grouping effect.
From Eq.~\ref{eq:conv},
$H^{(l+1)}$ is updated based on $H^{(0)}$ and $Z^{(l)*}H^{(l)}$.
Due to the grouping effect of $Z^{(0)*}$ and $H^{(0)}$,
the linear representation $Z^{(0)*}H^{(0)}$ also has grouping effect,
which further induces that
%we can conclude that 
$H^{(1)}$ has grouping effect.
In this way,
we can inductively prove that $Z^{(l)*}$, $(Z^{(l)*})^T$ and $H^{(l+1)}$ all have grouping effect.

%(1) $\lVert h_i^{(l)} - h_j^{(l)}\rVert_2 \rightarrow 0$ and (2) $ |\hat{A}_{{ip}}^k - \hat{A}_{{jp}}^k| \rightarrow 0,\ \forall k \in [1,K]$.
%%These imply $r = \bm{x}_i^T \bm{x}_j \rightarrow 1$ and  $|\hat{A}_{{ip}}^k - \hat{A}_{{jp}}^k| \rightarrow 0$.
%%Hence, the two terms of the numerator of the R.H.S of Equation~\ref{eq:norm} are close to 0. 
%The R.H.S. of Equation~\ref{eq:norm} will become close to 0, 
%which induces that $|Z_{ip}^*-Z_{jp}^*| \rightarrow 0$ and $Z^*$ thus has grouping effect.
\end{proof}

\section{Ablation study}
\label{sec:ab}
We next conduct an ablation study
to understand the main components of \ada.
%decouple neighborhood aggregation and feature transformation,
To construct the initial node embedding matrix $H^{(0)}$,
\ada\ first
transforms both feature matrix and adjacency matrix into low-dimensional embedding vectors, respectively.
To show the importance of feature matrix in constructing $H^{(0)}$,
we set $\alpha = 0$ in Equation~\ref{eq:h0} and derive $H^{(0)} = H_X^{(0)}$.
We call this variant \textbf{\ada-na} (no adjacency matrix).
Similarly,
to understand the importance of feature matrix,
we set $\alpha = 1$ and call the variant \textbf{\ada-nf} (no feature matrix).
Further,
to utilize the local structural information of a node,
\ada\ regularizes the coefficient matrix $Z$ with multi-hop graph adjacency matrices, as shown in Equation~\ref{eq:obj}.
We thus consider a variant \textbf{\ada-nl} (no local regularization)
by removing the regularization term to study
the importance of local graph structures of nodes.
Finally,
we compare \ada\ with these variants on all the benchmark datasets and show the results in Fig.~\ref{fig:ab}.
From the figure,
we see 

(1)
While \ada-na and \ada-nf can achieve comparable performance with \ada\ on some datasets,
\ada\ significantly outperforms them on others. 
This shows the necessity of 
\ada\ 
to adaptively learn the importance of feature matrix and adjacency matrix 
when constructing the initial node embedding vectors on various datasets.

(2)
\ada\ generally performs better than \ada-nl.
Since \ada-nl ignores the local regularization term,
it could fail to identify two homophilous nodes that share similar local graph structures.
On the other hand,
\ada\ measures node similarity in terms of both node features and local graph structures,
which further explains \ada's robustness towards graphs with various heterophilies.

%\begin{figure*}[!htbp]
%    \centering
%        \includegraphics[width = \linewidth]{fig/ablation.png}
%        \caption{Ablation study}
%        \label{fig:ab}
%\end{figure*}

\begin{figure*}[!htbp]
    \centering
        \includegraphics[width = \linewidth]{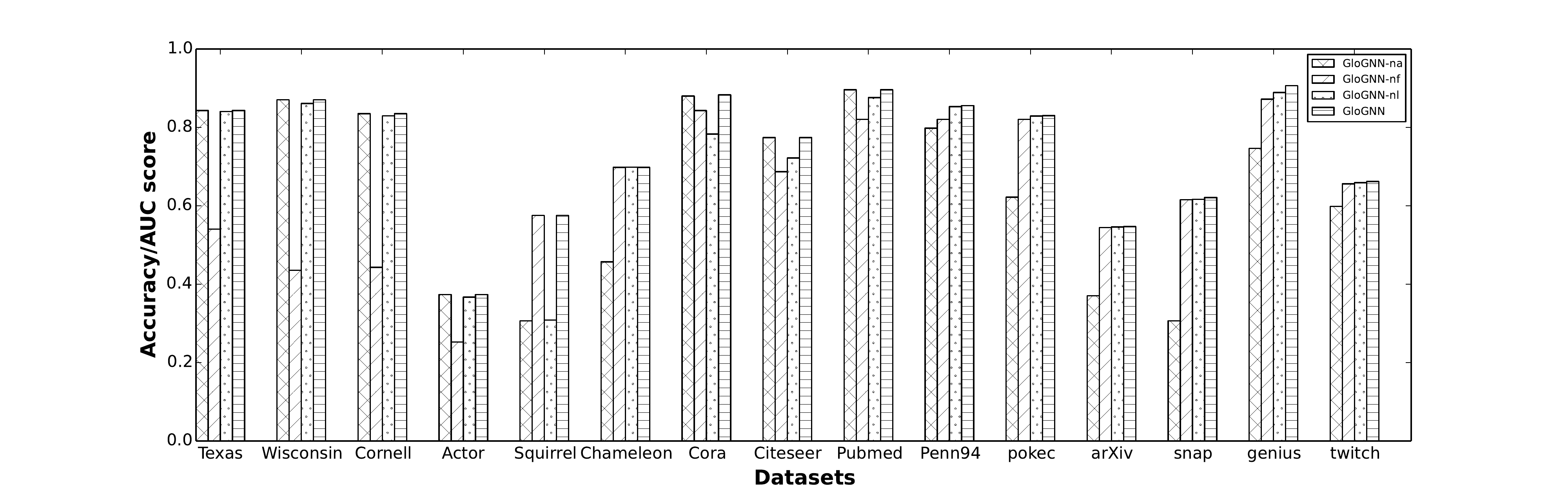}
        \caption{Ablation study}
        \label{fig:ab}
\end{figure*}

\section{Experimental setup}
\label{sec:setup}
We implemented \ada\ by PyTorch.
%The model 
%is
%%is initialized by Glorot initialization and 
%trained by Adam.
%We run the model for 3500 epochs with the learning rate 0.0005 on all the datasets.
For fairness,
we run the experiments of 9 small-scale datasets on CPUs 
and optimize the models by Adam as in~\cite{yan2021two}.
Meanwhile,
we run the experiments of 6 large-scale datasets on a single Tesla V100 GPU with 32G memory
and use AdamW as the optimizer following~\cite{lim2021large}.
We perform a grid search to tune hyper-parameters based on the results on the validation set.
Details of these hyper-parameters are listed in Tables~\ref{table:grid_search_small} and~\ref{table:grid_search_large}.
Further,
since the results of most baseline methods on these benchmark datasets are public,
we directly report these results.
For those cases where the results are absent,
we use the original codes released by their authors and fine tune the model parameters as suggested in~\cite{yan2021two,lim2021large,suresh2021breaking}.
%For WRGAT, we use default hyper-parameters to compute the multi-relational graph and fine tune other hyper-parameters following the original papers.
We provide our code and data at \url{https://github.com/RecklessRonan/GloGNN}.

%We list the grid search space for hyper-parameters as follows.
\begin{table}[!htbp]
\caption{Grid search space on small-scale datasets}
\centering
\begin{tabular}{|c|c|}
    \hline
    {\bf Notation} & {\bf Range} \\ \hline
     lr  & $\{0.01, 0.005\}$ \\ \hline 
      dropout & $[0, 0.9]$ \\ \hline 
       early\_stopping  & $\{40, 200, 300\}$ \\ \hline 
        weight\_decay  & $\{1\mathrm{e}{-5}, 5\mathrm{e}{-5}, 1\mathrm{e}{-4}\}$ \\ \hline 
         $\alpha$  & $[0,1]$ \\ \hline 
         $\beta_1$  & $\{0, 1, 10\}$ \\ \hline 
         $ \beta_2$  & $\{0.1, 1, 10, 10^2, 10^3\}$ \\ \hline  
           $\gamma$ & $[0, 0.9]$ \\ \hline 
            norm\_layers  & $\{1, 2, 3\}$ \\ \hline 
             max\_hop\_count $K$  & $[1, 6]$ \\ \hline 
\end{tabular}
\label{table:grid_search_small}
%\end{small}
\end{table}

\begin{table}[!htbp]
\caption{Grid search space on large-scale datasets}
\centering
\begin{tabular}{|c|c|}
    \hline
    {\bf Notation} & {\bf Range} \\ \hline
     lr  & $\{0.01, 0.005, 0.001\}$ \\ \hline 
      dropout & $[0, 0.9]$ \\ \hline 
        weight\_decay  & $\{0, 1\mathrm{e}{-3}, 1\mathrm{e}{-2}, 1\mathrm{e}{-1}\}$ \\ \hline 
         $\alpha$  & $\{0.1, 0.5, 0.9\}$ \\ \hline 
         $\beta_1$  & $\{0, 0.1, 1\}$ \\ \hline 
         $ \beta_2$  & $\{0.1, 1\}$ \\ \hline  
           $\gamma$ & $\{0.1, 0.5, 0.9\}$ \\ \hline 
            norm\_layers  & $\{1, 2, 3\}$ \\ \hline 
             max\_hop\_count $K$  & $\{1, 2, 3\}$ \\ \hline   
\end{tabular}
\label{table:grid_search_large}
%\end{small}
\end{table}

\end{document}